\documentclass{article}
\usepackage{microtype}
\usepackage{graphicx}
\usepackage{subcaption}
\usepackage{booktabs} 

\usepackage{hyperref}
\usepackage{titletoc}
\usepackage{enumitem}
\usepackage{makecell} 

\usepackage[accepted]{icml2026}

\usepackage{amsmath,amssymb}
\usepackage{bbm}

\usepackage{mathtools}
\usepackage{amsthm}
\usepackage{float}
\newcommand{\op}{\mathbin{\circ}}
\usepackage[capitalize,noabbrev]{cleveref}

\theoremstyle{plain}
\newtheorem{theorem}{Theorem}[section]

\newtheorem{lemma}[theorem]{Lemma}

\theoremstyle{definition}

\theoremstyle{remark}

\usepackage[textsize=tiny]{todonotes}

\begin{document}

\twocolumn[
  \icmltitle{CADFit: Precise Mesh-to-CAD Program Generation with Hybrid Optimization}



  \icmlsetsymbol{equal}{*}

  \begin{icmlauthorlist}
    \icmlauthor{Ghadi Nehme}{yyy}
    \icmlauthor{Eamon Whalen}{comp}
    \icmlauthor{Faez Ahmed}{yyy}
  \end{icmlauthorlist}

  \icmlaffiliation{yyy}{Department of Mechanical Engineering,  Massachusetts Institute of Technology, Cambridge, MA 02139, USA}
  \icmlaffiliation{comp}{Siemens Digital Industries Software, Plano, TX 75024, USA}

  \icmlcorrespondingauthor{Ghadi Nehme}{ghadi@mit.edu}

  \icmlkeywords{Machine Learning, ICML}

  \vskip 0.3in
]



\printAffiliationsAndNotice{}  

\begin{abstract}
    Despite recent progress, recovering parametric CAD construction sequences from geometric input, such as meshes or point clouds, is a key challenge for design and manufacturing, as existing CAD reconstruction and generation methods are largely restricted to difficult-to-edit formats like meshes or Breps or editable simple sketch-and-extrude pipelines and low-complexity datasets. We introduce \textbf{CADFit}, a hybrid optimization-based CAD reconstruction framework that recovers complex, editable CAD construction sequences from meshes by incrementally fitting and validating parametric operations using geometric feedback. Our approach is distinguished by formulating reconstruction as an IoU-driven optimization over structured CAD programs and supporting a rich set of operations, including extrusions, revolutions, fillets, and chamfers. Experiments on multiple CAD benchmarks show that CADFit outperforms state-of-the-art mesh-to-CAD methods in volumetric Intersection-over-Union and Chamfer Distance, while substantially reducing the Invalid Ratio of reconstructed CAD programs, particularly for complex designs. We further present a multimodal pipeline that enables end-to-end reconstruction of CAD construction sequences from images by combining image-based geometry reconstruction with CADFit. By enabling accurate reconstruction of higher-complexity CAD models, CADFit provides a practical foundation for generating richer datasets and advancing future learning-based approaches to CAD reverse engineering. The code is available at: \href{https://github.com/ghadinehme/CADFit}{https://github.com/ghadinehme/CADFit}.
\end{abstract}

\begin{figure*}[t]
    \centering
    \includegraphics[width=0.9\linewidth]{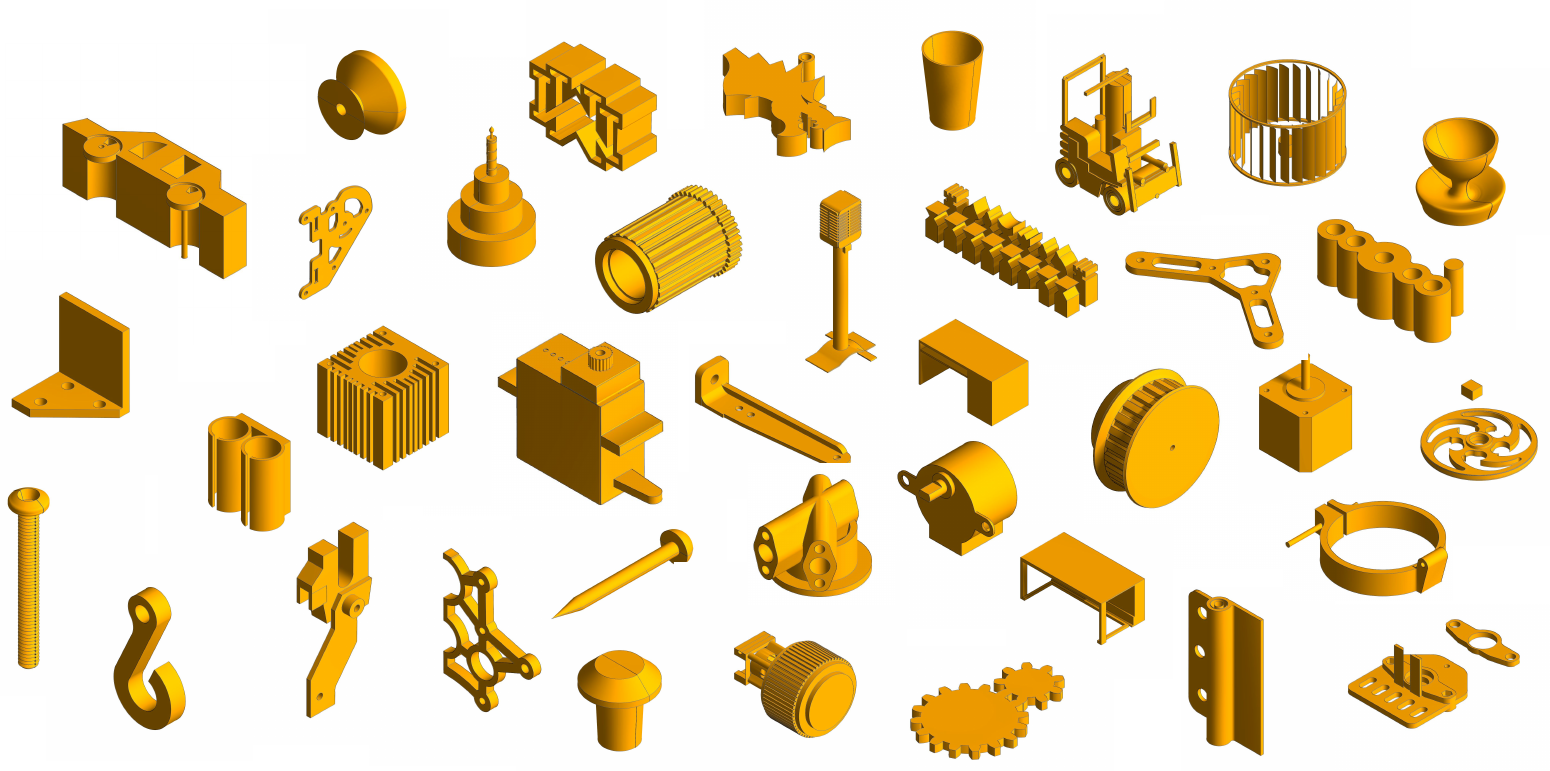}
    \caption{Mesh-to-CAD construction sequences generated by CADFit with varying complexity, including extrusions, revolutions, fillets, and chamfers.}
    \label{fig:showcase}
\end{figure*}

\section{Introduction}

Parametric CAD models represent geometry as structured construction sequences composed of geometric operations with continuous parameters. Such representations are widely used in engineering workflows because they enable compact shape specification and support downstream editing and constraint-based modeling \cite{camba2016parametric, li2010detecting}. Recovering parametric construction sequences from raw geometric data, however, remains a challenging inverse problem. A single target shape may admit many distinct but geometrically valid construction sequences \cite{lambourne2022reconstructing, willis2021fusion}, while each sequence must satisfy strict geometric constraints. Moreover, the problem requires jointly inferring discrete operation choices (e.g., feature type and ordering) and continuous parameters (e.g., dimensions and placements), making the search space highly non-convex and combinatorial.

Despite recent progress in CAD generation and reconstruction, existing AI-based methods remain limited in both operational expressiveness and geometric complexity. Most current CAD generation models focus on a narrow subset of operations, typically sketch-and-extrude pipelines, and are trained on datasets containing relatively simple shapes with few interacting features \cite{khan2024cad, alam2024gencad, wang2025cad}. As a consequence, these models struggle to represent realistic CAD designs involving multiple construction operations, complex boolean logic, and finishing operations such as fillets and chamfers. This limitation is not only a modeling issue but also a data bottleneck: the lack of accurate tools for reconstructing complex CAD construction sequences has constrained the scale and diversity of datasets available for learning.

In this paper, we introduce \textbf{CADFit}, a hybrid optimization-based method for reconstructing CAD construction sequences directly from geometry.
Given a watertight input mesh, CADFit recovers an ordered program composed of parametric operations drawn from a rich operator set, including extrusions, revolutions, fillets, and chamfers, and composed using Boolean union and cut.
Rather than predicting programs in a single forward pass, CADFit formulates reconstruction as an explicit optimization problem over executable CAD programs, guided by geometric feedback.

The core insight behind CADFit is that valid CAD programs can be recovered by tightly coupling geometric reasoning with program execution.
CADFit first extracts candidate sketch profiles from the input mesh, then generates a small set of geometrically consistent operation candidates by analyzing how parametric sweeps affect surface alignment.
These candidates are assembled into a construction sequence using an IoU-guided selection procedure that explicitly executes and validates each operation.
To reduce the combinatorial search space, CADFit additionally trains a sketch selection model that predicts which extracted sketch profiles are likely to contribute to the final construction sequence.
This learned prior is trained using labels automatically obtained from optimization-based reconstructions and is used only to filter candidates, without affecting correctness.

Importantly, CADFit decouples geometric perception from parametric reconstruction.
The method operates on surface geometry and can therefore be composed with upstream models that reconstruct meshes from images or point clouds.
This design enables a unified pipeline for mesh-, point-cloud-, and image-based CAD reconstruction without retraining or modifying the core algorithm.

We evaluate CADFit on multiple benchmarks, including DeepCAD \cite{wu2021deepcad}, Fusion360 Gallery \cite{willis2021fusion}, and ABC \cite{koch2019abc}, spanning a wide range of geometric complexity.
Across all benchmarks, CADFit achieves state-of-the-art performance in volumetric IoU and Chamfer Distance while maintaining zero invalid outputs.
Beyond geometric accuracy, CADFit recovers compact construction sequences that are often more concise than those found in existing CAD datasets, suggesting improved supervision for future learning-based approaches.

\noindent\textbf{Contributions.} We make the following contributions:
\begin{itemize}
  \item \textbf{Hybrid optimization for executable mesh-to-CAD programs.} We propose \textsc{CADFit}, an optimization-based framework that reconstructs \emph{executable} CadQuery/OpenCascade CAD construction sequences from a watertight input mesh by explicitly optimizing volumetric IoU under CAD-kernel validation, supporting extrude, revolve primitives, Boolean (union \& cut), and finishing operations (fillet \& chamfer).
  \item \textbf{Geometry-driven candidate generation and compact program search.} We introduce a reconstruction pipeline that (i) extracts sketch profiles from planar face clusters and slicing planes, (ii) generates a small set of high-quality extrude \& revolve candidates via one-sided Chamfer parameter sweeps and stable-interval selection, (iii) assembles compact sequences via IoU-guided backward pruning and iterative residual reconstruction with union \& cut, and (iv) accelerates search using a learned sketch prior trained from optimization-derived supervision.
  \item \textbf{Complexity-aware benchmarks and multimodal reconstruction.} We establish and evaluate on benchmarks spanning a wide range of CAD complexity (DeepCAD, Fusion360, and ABC stratified by complexity), and demonstrate consistent gains in IoU/Chamfer and near-zero invalid outputs for both mesh-to-CAD and an end-to-end image-to-CAD pipeline (image to mesh to \textsc{CADFit}).
\end{itemize}

\section{Related Work} 

\textbf{Learning-Based CAD Program Generation.}
There has been increasing interest in neural methods that generate or reconstruct CAD models as sequences of parametric operations.
Early approaches such as DeepCAD~\cite{wu2021deepcad} introduce a domain-specific language (DSL) for sketch-and-extrude programs and learn sequence autoencoders over this representation.
Subsequent methods extend this paradigm to richer inputs and outputs, including point clouds and meshes, and generate executable CAD code using structured decoders or code-generation models~\cite{xu2022skexgen,ren2022extrudenet,li2023secad, yu2025gencad,rukhovich2025cad,li2025caddreamer}.
While these methods produce editable outputs, they rely on amortized sequence prediction and often struggle with long, compositional programs, complex boolean interactions, and finishing operations.

Recent work further explores multimodal conditioning and large language models for CAD code generation~\cite{doris2025cad,kolodiazhnyi2025cadrille, alrashedy2025generating}.
These approaches broaden the range of inputs and tasks but inherit the brittleness of autoregressive program generation and do not explicitly optimize for geometric fidelity or program optimality.
In contrast, CADFit avoids direct sequence prediction and instead reconstructs construction programs through geometry-driven optimization.

\textbf{Geometry-Driven CAD Reverse Engineering and B-rep Reconstruction.}
Another line of work focuses on recovering CAD structure from geometric input using reverse engineering techniques.
Classical pipelines decompose shapes into primitives and fit analytic surfaces, but struggle to assemble consistent parametric models for complex designs.
Recent learning-assisted methods improve robustness by combining neural segmentation with analytic fitting, particularly for boundary representation (B-rep) reconstruction~\cite{jayaramansolidgen,liu2024point2cad,xu2024brepgen,lee2025brepdiff}.
These methods produce accurate surface topology and geometry, but do not recover ordered construction sequences or design intent encoded in parametric CAD programs.

Inverse CSG methods formulate reconstruction as a combinatorial or optimization-based search over boolean programs~\cite{du2018inversecsg, kania2020ucsg}.
Although such formulations improve validity guarantees, they typically operate over simple primitive sets and do not scale to sketch-based operations, revolutions, or finishing features.
CADFit differs by optimizing directly over parametric CAD operations while maintaining explicit geometric validation.

\textbf{Structured CAD Representations and Datasets.}
Large-scale datasets such as ABC, DeepCAD, and Fusion360 have enabled learning structured representations of CAD models, including sketches, feature trees, and operation graphs~\cite{koch2019abc,wu2021deepcad,willis2021fusion,xu2023hierarchical,uy2022point2cyl,man2026videocad,liu2024point2cad}.
More recent work studies neural representations of sketches, loops, and construction graphs to improve interpretability and controllability~\cite{xu2023hierarchical, karadeniz2025micadangelo}.
However, construction sequences in these datasets are often redundant or suboptimal, reflecting human modeling practices rather than canonical programs.

CADFit complements these efforts by reconstructing shorter, cleaner construction sequences through explicit optimization, which can serve as higher-quality supervision for future learning-based models.

\textbf{Multimodal 3D Reconstruction Pipelines.}
Advances in image- and point-based 3D reconstruction have enabled pipelines that combine perception models with downstream CAD inference.
Methods such as Pixel2Mesh~\cite{wang2018pixel2mesh}, Trellis ~\cite{xiang2025structured}, and Hunyuan3D~\cite{zhao2025hunyuan3d} recover high-quality surface geometry from images or sparse observations.
Several recent works combine such models with CAD reconstruction~\cite{li2025caddreamer}, but often require modality-specific adaptations or retraining.

CADFit adopts a decoupled design in which perception models recover surface geometry, while parametric reconstruction is performed via geometry-driven optimization, enabling robust CAD program recovery even from remeshed and smoothed reconstructed surfaces.

\begin{figure*}[t]
    \centering
    \includegraphics[width=\linewidth]{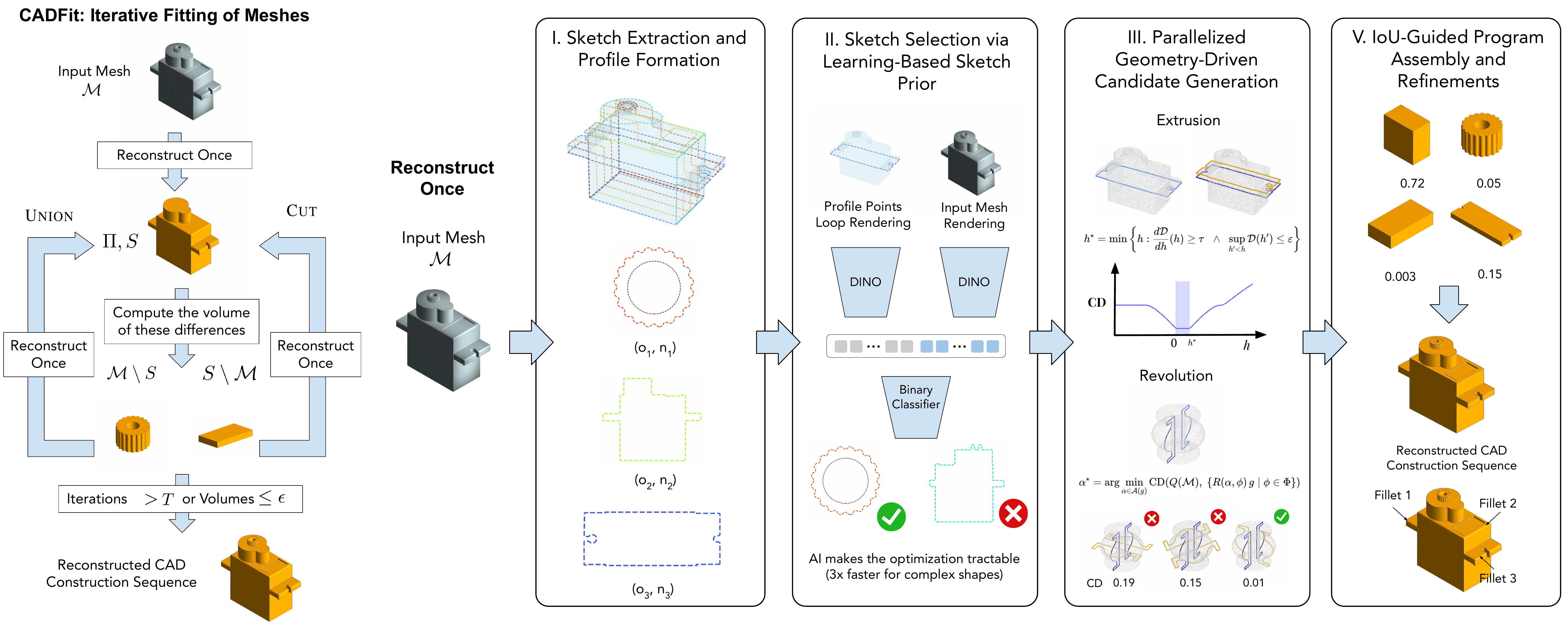}
    \caption{Overview of the CADFit Pipeline for Iterative Mesh-to-CAD Reconstruction}
    \label{fig:cadfit-pipeline}
\end{figure*}

\section{Method}
\label{sec:method}

We introduce \textbf{CADFit}, a hybrid optimization-based framework for recovering editable CAD construction programs from geometric input.
Given a watertight mesh $\mathcal{M} \subset \mathbb{R}^3$, CADFit reconstructs an ordered CAD program
\(
\Pi = (o_1,\dots,o_T)
\)
whose execution produces a solid closely approximating $\mathcal{M}$.
Each operation $o_t$ belongs to a fixed parametric operator set
\(
\mathcal{O} = \{\textsc{Extrude}, \textsc{Revolve}, \textsc{Fillet}, \textsc{Chamfer}, \textsc{Union}, \textsc{Cut}\}
\).
The complete procedure is summarized in Fig.~\ref{fig:cadfit-pipeline}.

\textbf{Problem Formulation.}
Let $\mathrm{Solid}(\Pi)$ denote the solid obtained by executing program $\Pi$.
We seek a valid program maximizing volumetric overlap with the target mesh:
\begin{equation}
\Pi^\star
=
\arg\max_{\Pi \in \mathcal{S}}
\mathrm{IoU}\big(\mathrm{Solid}(\Pi), \mathcal{M}\big),
\label{eq:objective}
\end{equation}
where $\mathcal{S}$ denotes the space of syntactically and geometrically valid CAD programs.
Boolean \textsc{Union} and \textsc{Cut} operations are applied sequentially to compose intermediate solids.

\textbf{CAD Program Representation.}
Each operation $o_t \in \Pi$ is represented as a tuple $o_t = (\tau_t, \theta_t, R_t)$, where $\tau_t$ is the operation type,
\(
\tau_t \in \{\textsc{Extrude},\textsc{Revolve},\textsc{Fillet},\textsc{Chamfer},\textsc{Union},\textsc{Cut}\},
\)
$\theta_t$ denotes the continuous operation parameters (e.g., extrusion height, revolution angle and axis, fillet radius, chamfer length),
and $R_t$ specifies the geometric references required by the operation.
Depending on $\tau_t$, $R_t$ may refer to a sketch profile, a set of edges, or one or more existing solids.

We implement CADFit using \textbf{CadQuery} \cite{cadquery_contributors_2025_14590990}, an OpenCascade-based parametric CAD library.
CAD programs $\Pi$ are represented directly using the CadQuery operation grammar, including \texttt{extrude}, \texttt{revolve}, \texttt{fillet}, \texttt{chamfer}, \texttt{union}, and \texttt{cut}, with sketches defined on explicit planes using line, arc, and circle primitives (Fig. \ref{fig:cadquery}).
Kernel feedback is used to reject invalid operations; operations that fail to execute are discarded.

\textbf{Pipeline Overview.}
CADFit reconstructs a CAD program through iterative mesh fitting. 
The method consists of an outer residual-reconstruction loop and an inner single-pass reconstruction procedure.
Given an input mesh $\mathcal{M}$, CADFit first runs a single-pass reconstruction to obtain an initial CAD program $\Pi_0$ and reconstructed solid $S_0$.
It then computes two complementary residuals: the under-reconstructed region
\(
R_k^{+} = \mathcal{M} \setminus S_k
\),
which contains target geometry missing from the current reconstruction, and the over-reconstructed region
\(
R_k^{-} = S_k \setminus \mathcal{M}
\),
which contains geometry present in the reconstruction but absent from the target.
CADFit applies the same single-pass reconstruction procedure to these residual meshes and combines the recovered residual programs with the current program using Boolean operations: under-reconstructed residuals are added with \textsc{Union}, while over-reconstructed residuals are removed with \textsc{Cut}.
This process can be repeated for a bounded number of iterations, progressively refining the CAD program.

The inner single-pass reconstruction procedure has four stages.
First, CADFit extracts candidate sketch profiles from the input mesh using planar face clusters and axis-aligned slicing planes.
Second, a learned sketch prior filters the extracted profiles to reduce the number of noisy or irrelevant candidates before expensive CAD-kernel evaluation.
Third, CADFit performs geometry-driven candidate generation by fitting \textsc{Extrude} and \textsc{Revolve} operations to the retained sketch profiles through parameter sweeps.
Finally, CADFit assembles a compact program in two stages: it first greedily selects candidates that improve IoU with the target mesh, and then applies IoU-guided backward pruning to remove candidates whose deletion does not reduce reconstruction quality.
This nested design allows CADFit to recover complex shapes by repeatedly fitting simpler residual geometries while keeping each single-pass reconstruction tractable and compact. (Figure \ref{fig:cadfit-pipeline})

\begin{figure*}[t]
    \centering
    \includegraphics[width=1\linewidth]{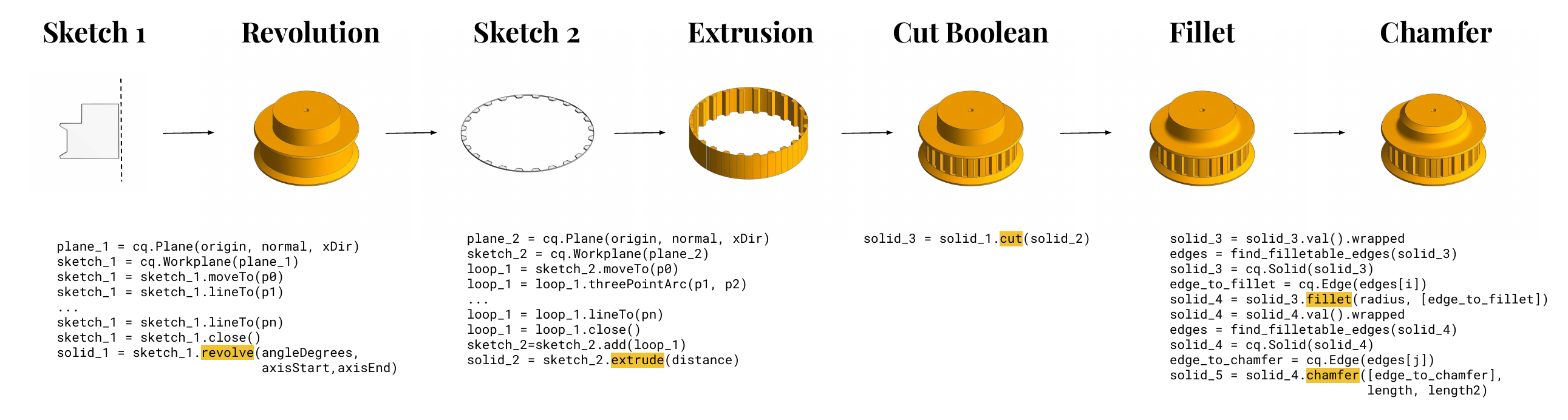}
    \caption{Executable CadQuery program reconstructed by CADFit, illustrating sketch-based primitives, extrusion and revolution operations, Boolean cut and union, and subsequent fillet and chamfer features.}
    \label{fig:cadquery}
\end{figure*}

\textbf{Single-Pass Reconstruction.}
The single-pass reconstruction procedure is the basic fitting primitive used both for the initial mesh and for residual meshes in later iterations.
Given a target mesh, it extracts sketch profiles, filters them with the learned sketch prior, generates parametric \textsc{Extrude} and \textsc{Revolve} candidates, and assembles a compact program by greedily selecting IoU-improving candidates followed by IoU-guided backward pruning.

\textbf{Sketch Extraction and Profile Formation.}
CADFit constructs a finite set of sketch profiles $\mathcal{G}=\{g_i\}$ from $\mathcal{M}$ using explicit geometric operators.
Candidate sketch planes are defined by (i) detected planar face clusters and (ii) a fixed set of axis-aligned slicing planes.
For each plane $(\mathbf{o},\mathbf{n})$, we compute local mesh sections and project resulting intersection curves onto the plane, yielding ordered 2D loops.
Degenerate loops are discarded, and remaining loops are grouped by polygon containment to form sketch profiles with outer boundaries and interior holes.
All geometric extraction steps are detailed in Appendix~\ref{app:sketchextraction}.

\textbf{Learning-Based Sketch Prior.}
To keep the sketch profile set $\mathcal{G}$ tractable, CADFit uses a learned prior
\(
f(\mathcal{M}, g) \in [0,1],
\)
which estimates whether a candidate profile $g$ is useful for reconstructing mesh $\mathcal{M}$.

The prior is trained from CADFit's pure optimization pipeline.
Profiles used in the recovered CAD program are labeled positive, while the remaining extracted profiles are labeled negative.
Instead of encoding the mesh and sketch separately, we form a joint visual input: CADFit renders multi-view images of $\mathcal{M}$ and overlays the candidate loop $g$ on each view.
These images are passed through a DINO encoder \cite{oquab2023dinov2} to obtain a global visual embedding, which is then fed to a lightweight 3-layer MLP with ReLU activations and a sigmoid output.
The resulting score
\(
p(g\mid\mathcal{M}) = f(\mathcal{M}, g)
\)
measures mesh--sketch compatibility.

At inference time, each profile is independently selected with probability $p(g\mid\mathcal{M})$.
This converts the prior into a stochastic filter: it favors sketches likely to appear in the final CAD program while preserving exploration.
The selected profiles are then fit with sketch primitives.
Additional details are provided in Appendix~\ref{app:learnedprior}.

\textbf{Geometry-Driven Candidate Generation.}
For each sketch profile $g \in \mathcal{G}$, CADFit generates a small discrete set of \textsc{Extrude} and \textsc{Revolve} candidates by solving a continuous parameter selection problem.
Let $Q(\mathcal{M})$ denote a surface point cloud sampled from $\mathcal{M}$.
For a candidate operation $o(g,\theta)$, we define a one-sided Chamfer objective
\begin{equation*}
\mathcal{D}(\theta)
=
\frac{1}{|P(g,\theta)|}
\sum_{\mathbf{x}\in P(g,\theta)}
\min_{\mathbf{y}\in Q(\mathcal{M})}
\|\mathbf{x}-\mathbf{y}\|_2^2,
\end{equation*}
where $P(g,\theta)$ are surface points induced by applying the operation with parameters $\theta$ to profile $g$.
Rather than minimizing $\mathcal{D}$ directly, we sweep $\theta$ over a bounded domain and identify stable parameter intervals preceding sharp increases in $\mathcal{D}$.
These intervals correspond to geometrically consistent operations that do not overshoot the target shape.
All parameter selection details are provided in Appendix~\ref{app:paramsearch}.

\textbf{IoU-Guided Program Assembly.}
Let $\mathcal{C}=\{c_j\}$ denote the set of all generated operation candidates.
CADFit assembles a base program using a forward greedy selection stage followed by backward pruning.
Starting from an empty program $\Pi=\emptyset$ and solid $S=\emptyset$, CADFit repeatedly adds the candidate that gives the largest positive IoU improvement:
\[
c^\star =
\arg\max_{c\in\mathcal{C}\setminus\Pi}
\left[
\mathrm{IoU}(S\cup c,\mathcal{M})
-
\mathrm{IoU}(S,\mathcal{M})
\right].
\]
If this improvement is positive, $c^\star$ is added to $\Pi$ and merged into $S$.
Forward selection terminates when no remaining candidate improves IoU.
After this forward selection stage, CADFit applies backward marginal pruning to remove redundant candidates.
Starting from the greedily selected program, candidates are iteratively removed according to their marginal removal effect
\[
\Delta^{-}(c\mid S)
=
\mathrm{IoU}(\mathrm{Solid}(\Pi \setminus \{c\}),\mathcal{M})
-
\mathrm{IoU}(S,\mathcal{M}),
\]
removing at each step the candidate with the largest $\Delta^{-}$.
Pruning stops when removing any remaining candidate would decrease IoU.
The resulting subset is compact and contains only candidates that contribute to reconstruction quality under union composition. The complete forward-selection and backward-pruning procedure is provided in Algorithm~\ref{alg:search} in Appendix~\ref{app:assembly_algo}.

\textbf{Iterative Residual Reconstruction.}
After the initial single-pass reconstruction, CADFit refines the program by repeatedly fitting residuals between the current reconstruction and the target mesh.
At iteration $k$, the under-reconstructed residual
\(
R_k^{+}=\mathcal{M}\setminus S_k
\)
captures missing target geometry, while the over-reconstructed residual
\(
R_k^{-}=S_k\setminus\mathcal{M}
\)
captures excess reconstructed geometry.
Each residual is reconstructed using the same single-pass procedure: the program recovered from $R_k^{+}$ is added to the current reconstruction with \textsc{Union}, while the program recovered from $R_k^{-}$ is subtracted with \textsc{Cut}.
Iterations terminate when residual volumes fall below a threshold $\epsilon$ or when a fixed iteration budget is reached.
The full procedure is given in Algorithm~\ref{alg:iterative}.

\begin{algorithm}[H]
  \caption{Iterative Residual CAD Program Fitting}
  \label{alg:iterative}
  \begin{algorithmic}
    \STATE {\bfseries Input:} mesh $\mathcal{M}$, tolerance $\epsilon$, max iterations $T$
    \STATE {\bfseries Output:} CAD program $\Pi^\star$
    \STATE $(\Pi,S)\leftarrow\textsc{ReconstructOnce}(\mathcal{M})$
    \FOR{$t=1$ {\bfseries to} $T$}
        \STATE $R^{+}\leftarrow\mathcal{M}\setminus S$, $R^{-}\leftarrow S\setminus\mathcal{M}$
        \IF{$\mathrm{Vol}(R^{+})\le\epsilon$ {\bfseries and} $\mathrm{Vol}(R^{-})\le\epsilon$}
            \STATE \textbf{break}
        \ENDIF
        \STATE Reconstruct $R^{+}$ and $R^{-}$ and combine via \textsc{Union}/\textsc{Cut}
    \ENDFOR
    \STATE \textbf{return} $\Pi^\star$
  \end{algorithmic}
\end{algorithm}

\textbf{Application to Multimodal CAD Reconstruction.} CADFit is formulated as a mesh-to-CAD reconstruction method but is agnostic to the source of the mesh.
By treating geometry as an intermediate representation, CADFit can be composed with upstream perception models that reconstruct surface geometry from images or point clouds.
For image-based inputs, we obtain a surface mesh using a pretrained image-to-3D model and apply standard post-processing to satisfy CADFit’s geometric assumptions.
The processed mesh is then passed directly to CADFit without modification of the reconstruction pipeline. Implementation details are provided in Appendices \ref{app:image2mesh} and \ref{app:trellis-hunyuan}. Importantly, CADFit is not retrained or adapted for different input modalities.
The same optimization procedure is applied unchanged to meshes originating from images, point clouds, or CAD exports, enabling a clean separation between perception and parametric reconstruction.
We evaluate the resulting end-to-end pipeline in Section~\ref{sec:experiments}.

\section{Experiments}
\label{sec:experiments}

We evaluate CADFit on two settings: \textbf{mesh-to-CAD} and \textbf{image-to-CAD}.
Across both settings, the goal is to recover an editable CAD construction program whose execution yields a solid that matches the target.

\subsection{Tasks and Baselines}
\textbf{Mesh-to-CAD.}
Given a watertight input mesh $\mathcal{M}$, methods must output a CAD construction sequence (or an equivalent executable program) that reconstructs $\mathcal{M}$.
We compare against three baselines: GenCAD-3D~\cite{yu2025gencad}, CAD-Recode~\cite{rukhovich2025cad}, and Cadrille~\cite{kolodiazhnyi2025cadrille}.
When a baseline does not support mesh input, we follow its standard preprocessing and provide point-cloud input.

\textbf{Image-to-CAD.}
Given a single RGB image, we first reconstruct an intermediate surface mesh using a pretrained image-to-3D model. We adopt Hunyuan3D \cite{zhao2025hunyuan3d} after comparing it against another state-of-the-art image-to-mesh method, Trellis \cite{xiang2025structured}, with an ablation study reported in Appendix~\ref{tab:abc_hunyuan_trellis_stats}. The reconstructed mesh is then post-processed to satisfy CADFit’s geometric assumptions by enforcing watertightness, applying Taubin smoothing \cite{taubin1995curve} to suppress high-frequency artifacts, and performing mesh decimation \cite{garland1997surface}.
All mesh-to-CAD baselines (GenCAD-3D, CAD-Recode, Cadrille) are run on these processed meshes.
We additionally compare against a direct image-to-CAD program generation baseline CAD-Coder~\cite{doris2025cad}.

\begin{figure*}[t]
    \centering
    \includegraphics[width=1\linewidth]{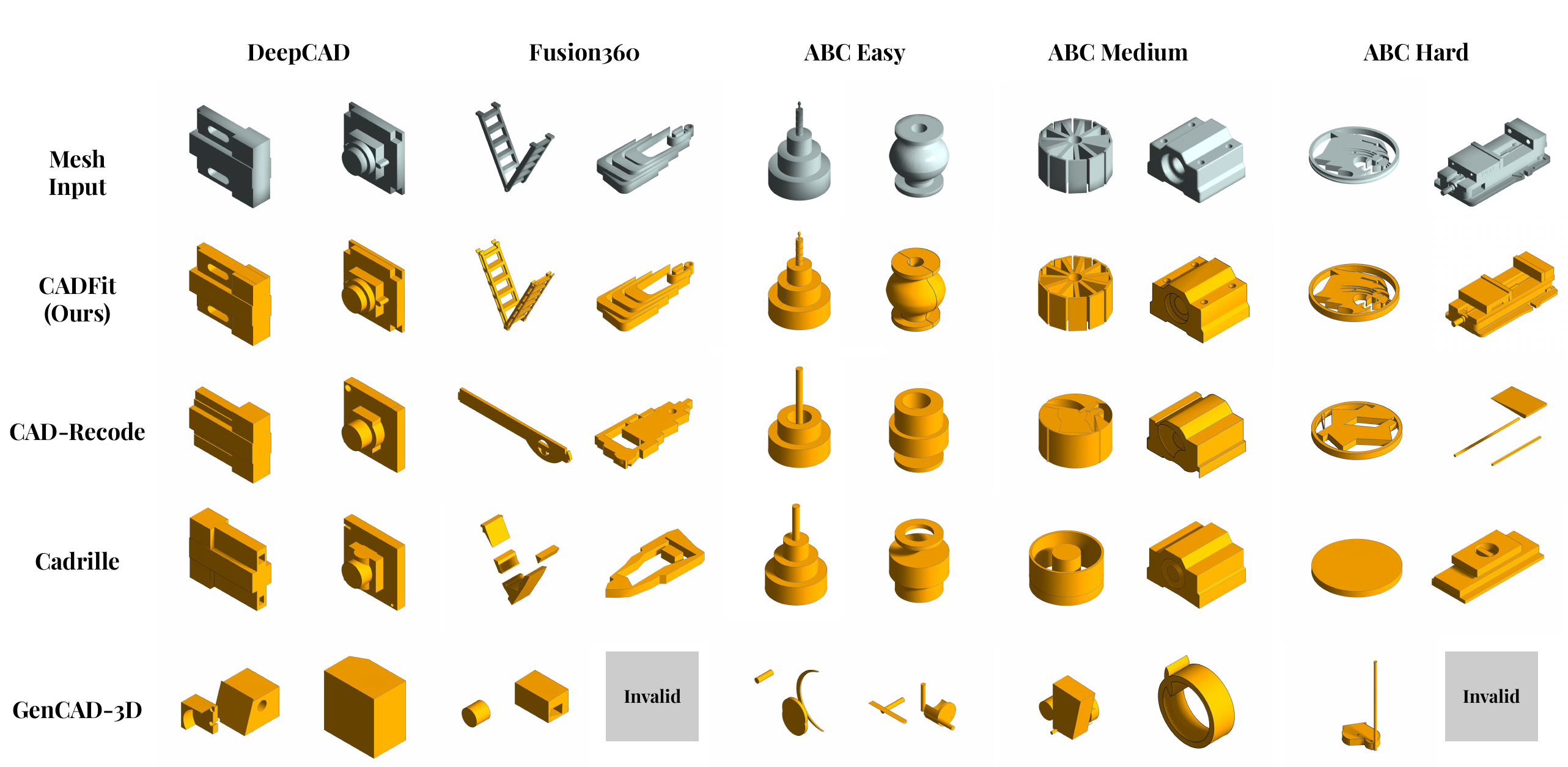}
    \caption{Qualitative mesh-to-CAD reconstruction results on DeepCAD, Fusion360, and ABC (easy, medium and hard subsets).}
    \label{fig:results-mesh}
\end{figure*}

\begin{table*}[t]
\centering
\caption{CAD construction sequence reconstruction results on DeepCAD and Fusion360, and on ABC stratified by CAD complexity. Improvement is computed relative to the strongest baseline in each column (max IoU / min CD / min IR).}
\label{tab:pc_mesh_results}
\resizebox{\textwidth}{!}{
\begin{tabular}{lccc ccc ccccccccc}
\toprule
 & \multicolumn{3}{c}{\textbf{DeepCAD}} 
 & \multicolumn{3}{c}{\textbf{Fusion360}}
 & \multicolumn{9}{c}{\textbf{ABC}} \\
\cmidrule(lr){2-4}
\cmidrule(lr){5-7}
\cmidrule(lr){8-16}

\textbf{Method}
& IoU$\uparrow$ & CD$\downarrow$ & IR$\downarrow$
& IoU$\uparrow$ & CD$\downarrow$ & IR$\downarrow$
& \multicolumn{3}{c}{Easy}
& \multicolumn{3}{c}{Medium}
& \multicolumn{3}{c}{Hard} \\
\cmidrule(lr){8-10}
\cmidrule(lr){11-13}
\cmidrule(lr){14-16}

 &
 &  &
 &  &  &
 & IoU$\uparrow$ & CD$\downarrow$ & IR$\downarrow$
 & IoU$\uparrow$ & CD$\downarrow$ & IR$\downarrow$
 & IoU$\uparrow$ & CD$\downarrow$ & IR$\downarrow$ \\
\midrule

GenCAD-3D
& 0.255 & 0.132 & 0.020
& 0.208 & 0.143 & 0.313
& 0.258 & 0.141 & 0.263
& 0.204 & 0.135 & 0.377
& 0.135 & 0.128 & 0.390 \\

CAD-Recode 
& 0.818 & 0.027 & 0.000
& 0.761 & 0.027 & 0.003
& 0.773 & 0.030 & 0.003
& 0.571 & 0.044 & 0.020
& 0.333 & 0.054 & 0.010 \\

Cadrille
& 0.872 & 0.020 & 0.007
& 0.817 & 0.019 & 0.003
& 0.808 & 0.028 & 0.003
& 0.653 & 0.032 & 0.017
& 0.438 & 0.040 & 0.013 \\

\midrule
\textbf{CADFit (Ours)}
& \textbf{0.964} & \textbf{0.016} & \textbf{0.000}
& \textbf{0.921} & \textbf{0.017} & \textbf{0.000}
& \textbf{0.908} & \textbf{0.022} & \textbf{0.000}
& \textbf{0.851} & \textbf{0.026} & \textbf{0.000}
& \textbf{0.662} & \textbf{0.032} & \textbf{0.000} \\

\textbf{Improvement (\%)}
& $\uparrow$ 10.6 & $\downarrow$ 19.3 & 0.0
& $\uparrow$ 12.7 & $\downarrow$ 10.2 & $\downarrow$ 100.0
& $\uparrow$ 12.4 & $\downarrow$ 21.4 & $\downarrow$ 100.0
& $\uparrow$ 30.3 & $\downarrow$ 18.8 & $\downarrow$ 100.0
& $\uparrow$ 51.1 & $\downarrow$ 20.0 & $\downarrow$ 100.0 \\

\bottomrule
\end{tabular}
}
\end{table*}

\begin{table*}[t]
\centering
\footnotesize
\caption{Image-to-CAD construction sequence reconstruction results on the ABC dataset subsets easy, medium and hard.}
\label{tab:image_to_cad_results}
\begin{tabular}{lccccccccc}
\toprule
 & \multicolumn{9}{c}{\textbf{ABC}} \\
\cmidrule(lr){2-10}

\textbf{Method}
& \multicolumn{3}{c}{Easy}
& \multicolumn{3}{c}{Medium}
& \multicolumn{3}{c}{Hard} \\
\cmidrule(lr){2-4}
\cmidrule(lr){5-7}
\cmidrule(lr){8-10}

 &
IoU$\uparrow$ & CD$\downarrow$ & IR$\downarrow$
& IoU$\uparrow$ & CD$\downarrow$ & IR$\downarrow$
& IoU$\uparrow$ & CD$\downarrow$ & IR$\downarrow$ \\
\midrule

GenCAD-3D
& 0.310 & 0.104 & 0.047
& 0.246 & 0.112 & 0.107
& 0.169 & 0.096 & 0.207 \\

CAD-Coder
& 0.405 & 0.086 & 0.013
& 0.2912 & 0.103 & 0.027
& 0.1710 & 0.104 & 0.053 \\

Cadrille
& 0.438 & 0.103 & 0.027
& 0.385 & 0.113 & 0.000
& 0.226 & 0.106 & 0.020 \\

CAD-Recode
& 0.437 & 0.063 & 0.000
& 0.391 & 0.071 & 0.000
& 0.237 & 0.074 & 0.040 \\

\midrule
\textbf{CADFit (Ours)}
& \textbf{0.669} & \textbf{0.036} & \textbf{0.000}
& \textbf{0.545} & \textbf{0.037} & \textbf{0.000}
& \textbf{0.380} & \textbf{0.041} & \textbf{0.000} \\

\textbf{Improvement (\%)}
& $\uparrow$ 52.9 & $\downarrow$ 42.9 & 0.0
& $\uparrow$39.4 & $\downarrow$ 47.9 & 0.0
& $\uparrow$60.3 & $\downarrow$ 44.6 & $\downarrow$ 100.0 \\
\bottomrule
\end{tabular}
\end{table*}

\subsection{Benchmarks}
We construct benchmarks spanning a wide range of geometric complexity.
Throughout, we use \emph{STEP face count} as a proxy for CAD complexity (a common and effective complexity measure \cite{contero2023quantitative}).

\textbf{DeepCAD (complex subset).}
We evaluate on the official DeepCAD test split and select the most complex CAD models according to STEP face count to avoid over-representing trivial primitives (e.g., cylinders and boxes).
The resulting DeepCAD test set contains 300 shapes.

\textbf{Fusion360 Gallery (uniform face-count).}
We sample 300 CAD models from Fusion360 Gallery such that the STEP face count is approximately uniform over a predefined range.
This benchmark stresses reconstruction across diverse geometries while controlling for complexity distribution.

\textbf{ABC (stratified by complexity).}
ABC contains shapes spanning from very simple to extremely complex CAD.
We partition ABC into three complexity windows using STEP face count:
$\text{Easy: } [1,15],\quad
\text{Medium: } [16,150],\quad
\text{Hard: } [151,1500],
$
and sample 300 shapes uniformly over face count within each window.
Shapes with more than 1500 faces are excluded to avoid extreme outliers.

\textbf{ABC Image-to-CAD benchmark.}
For image-to-CAD, we sample 150 shapes from each ABC window (Easy/Medium/Hard), render an isometric view per shape using Blender, and use these images as inputs (Fig. \ref{fig:abc-images}).

\subsection{Evaluation Metrics}

We evaluate using Chamfer Distance (CD), volumetric Intersection-over-Union (IoU), and Invalid Ratio (IR).
Predictions are rigidly aligned with a global scale to ground truth by minimizing symmetric CD prior to computing CD and IoU.
CD is the symmetric Chamfer Distance between surface point clouds, with ground-truth meshes normalized to $[-1,1]^3$.
IoU is the Intersection over Union between the generated and ground truth solids:
$\mathrm{IoU}(A,B)=\mathrm{Vol}(A\cap B)/\mathrm{Vol}(A\cup B)$.
IR measures the fraction of invalid outputs (e.g., failed CAD program execution).
Details are in Appendix~\ref{app:metrics}.

\subsection{Results and Discussion}
\textbf{Mesh-to-CAD.}
Table~\ref{tab:pc_mesh_results} reports quantitative mesh-to-CAD results on DeepCAD, Fusion360, and ABC, stratified by CAD complexity measured using STEP face count.
Across all benchmarks, CADFit consistently outperforms prior methods in IoU, Chamfer Distance, and Invalid Ratio.

Learning-based baselines such as GenCAD-3D and Cadrille exhibit limited generalization beyond their training distributions.
While these methods achieve reasonable performance on datasets similar to their training data, their reconstruction quality degrades sharply on more diverse benchmarks such as Fusion360 and on higher-complexity ABC subsets.
This behavior is particularly evident in the Hard ABC split, where GenCAD-3D and Cadrille frequently oversimplify geometry or fail to recover a valid sequence.

In contrast, CADFit maintains strong performance across datasets and complexity levels.
On the ABC Hard subset, CADFit more than doubles the median IoU of the strongest prior method (CAD-Recode) (Figure \ref{tab:pc_mesh_results_median}), demonstrating its ability to recover long and interacting construction sequences.
These gains highlight the importance of geometry-driven optimization and explicit program validation when reconstructing complex CAD models.
Qualitative comparisons in Fig.~\ref{fig:results-mesh} further show that CADFit more faithfully recovers thin structures, rotational primitives, and multi-step compositions that are often missed by purely learning-based approaches.

\textbf{Image-to-CAD.}
Table~\ref{tab:image_to_cad_results} reports end-to-end image-to-CAD results on the ABC subsets.
Despite noise and artifacts introduced during image-based reconstruction, the Hunyuan3D + preprocessing + CADFit pipeline achieves the best performance across all complexity levels.
This result underscores the advantage of decomposing image-to-CAD reconstruction into a multi-stage process, where each component addresses a well-defined subproblem: perception models recover approximate surface geometry, while CADFit performs structured parametric reconstruction.
Additional results and visualizations can be found in Appendix \ref{app:more-results}.

\begin{figure}[t]
    \centering
    \includegraphics[width=1\linewidth]{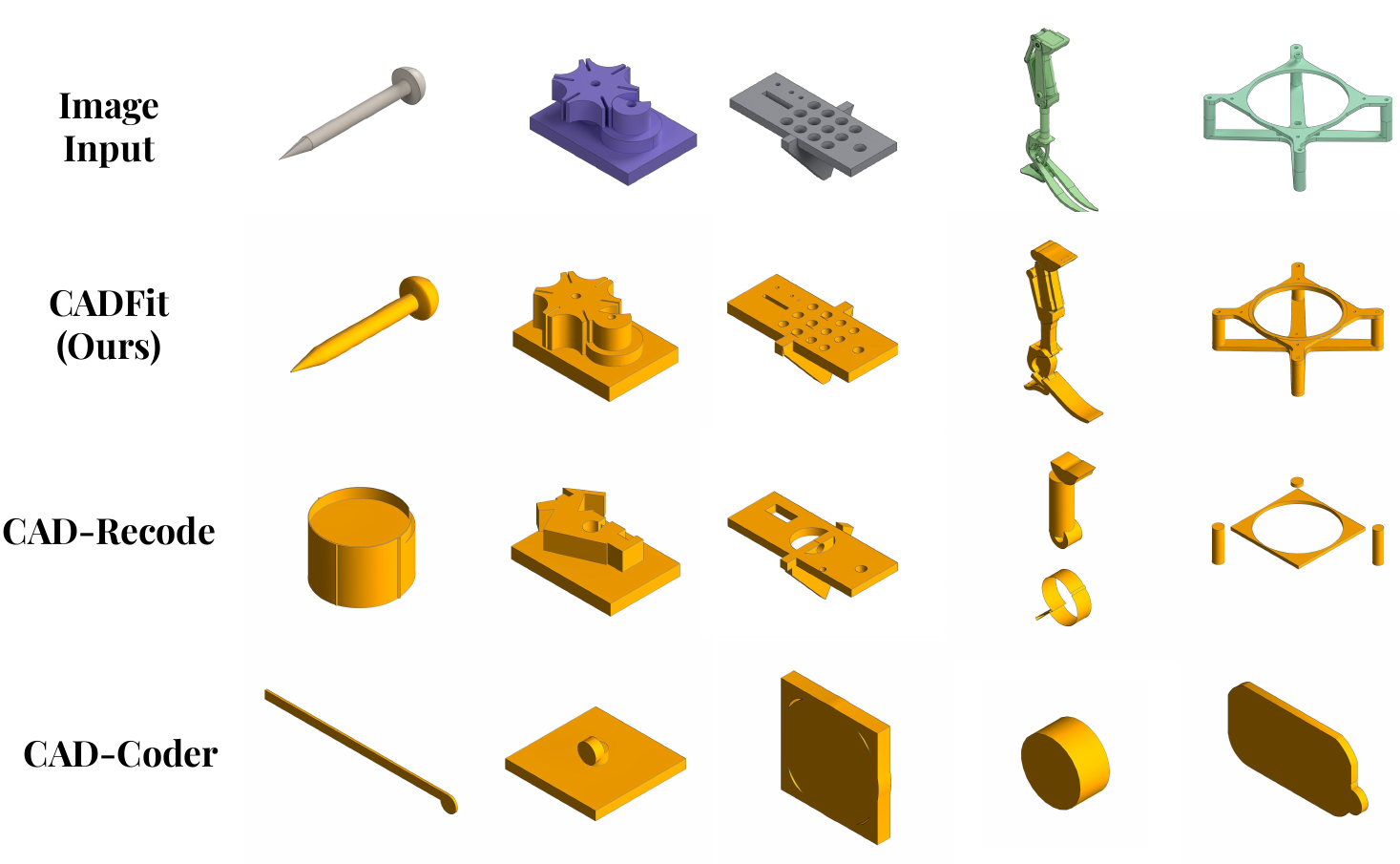}
    \caption{Qualitative image-to-CAD reconstruction results comparing CADFit with prior methods.}
    \label{fig:results-image}
\end{figure}

\textbf{Construction Sequence Quality.}
Beyond geometric accuracy, CADFit recovers compact construction sequences.
As shown in Table~\ref{tab:cadfit_win_breakdown}, in \textbf{95.1\%} of cases CADFit achieves higher IoU than competing methods, and in \textbf{69.8\%} of cases it simultaneously attains the highest IoU and the shortest valid sequence.
Moreover, in \textbf{85\%} of cases CADFit generates construction sequences that are shorter than the ground-truth DeepCAD programs while achieving \textbf{IoU $> 0.9$} (Figure \ref{fig:deepcad-cadfit}).

\begin{table}[H]
\centering
\footnotesize
\caption{Breakdown of CADFit wins across IoU and construction sequence length against Cadrille and CAD-Recode on the complex DeepCAD subset (N = 300).}
\label{tab:cadfit_win_breakdown}
\begin{tabular}{lccc}
\toprule
\textbf{Outcome} & \textbf{Percentage}\\
\midrule
Wins both IoU and length
 & 69.8\%\\

Wins IoU only
& 25.3\%\\

Wins length only
& 3.9\%\\

Wins neither
& 1.1\% \\

\bottomrule
\end{tabular}
\end{table}

This behavior arises from CADFit’s optimization-driven formulation, which retains only operations that improve geometric fidelity rather than replicating potentially redundant or suboptimal construction histories.
As a result, the recovered programs are easier to interpret and edit. Such programs may serve as improved training data for learning-based CAD models or as structured exploration policies for reinforcement learning approaches to CAD design.

\textbf{Contribution of the Learned Sketch Prior.}
CADFit uses learning only to filter candidate sketches, but this filter is essential for tractability.
Initial extraction yields hundreds of plausible profiles, many noisy, fragmented, or irrelevant.
Evaluating all of them with CAD-kernel calls is costly, so the no-prior baseline randomly samples 100 loops to keep optimization tractable.
The learned prior replaces this random budget with a ranked subset of the most useful sketches.

\begin{table}[H]
\centering
\caption{
Effect of the learned sketch prior on quality and runtime across ABC splits.
Under the same tractable candidate budget, the prior selects useful profiles instead of random loops, improving both scalability and reconstruction quality.
}
\label{tab:runtime_scaling}
\resizebox{\linewidth}{!}{
\begin{tabular}{lcccc}
\toprule
\textbf{ABC subset} 
& \textbf{IoU change} 
& \textbf{Speedup} 
& \makecell{\textbf{Time w/o}\\\textbf{prior (s)}} 
& \makecell{\textbf{Time w/}\\\textbf{prior (s)}} \\
\midrule
Easy   & $-0.0033$ & $1.42{\times}$ & $303.0$ & $213.3$ \\
Medium & $+0.0032$ & $1.62{\times}$ & $396.1$ & $244.4$ \\
Hard   & $+0.2285$ & $2.90{\times}$ & $903.7$ & $311.8$ \\
\bottomrule
\end{tabular}
}
\end{table}

The prior ranks profiles by their likelihood of appearing in the final CAD program.
Table~\ref{tab:runtime_scaling} shows speedups from $1.42{\times}$ on ABC-easy to $2.90{\times}$ on ABC-hard.
On the hardest split, it also raises IoU by $+0.2285$ because the prior selects high-value sketches within the candidate budget, whereas the no-prior setting samples loops randomly.
Thus, learning in CADFit acts as a guide rather than a generator: it preserves geometric optimization while replacing random search with a smaller, higher-quality candidate set.

\textbf{Robustness to Noisy Meshes and Preprocessing.}
CADFit works best on clean, watertight meshes, since residual refinement uses boolean complements. Image-based reconstructions, however, are often noisy, dense, non-watertight, or topologically flawed. CADFit can still run its initial pass on such inputs, but full residual refinement may require preprocessing. Noise mainly increases cost rather than causing direct failure. It creates fragmented sketch candidates and slows IoU scoring. Simple preprocessing, such as Taubin smoothing and decimation, improves both speed and quality. Table~\ref{tab:hunyuan_preprocess} shows this on Hunyuan3D meshes. Thus, CADFit is compatible with image-to-3D meshes, but preprocessing helps improve speed and robustness.

\begin{table}[H]
\centering
\caption{
Effect of lightweight preprocessing on image-based meshes.
}
\label{tab:hunyuan_preprocess}
\small
\begin{tabular}{lcc}
\toprule
\textbf{Input setting} & \textbf{IoU} & \textbf{Runtime (s)} \\
\midrule
Noisy mesh              & $0.876$ & $1255$ \\
Smoothed + decimated    & $0.914$ & $456$ \\
\bottomrule
\end{tabular}
\end{table}

\textbf{Hyperparameter Selection.}
To justify the default configuration used in our main experiments, we ablate key hyperparameters controlling candidate generation, sketch extraction, residual refinement, and parameter search. 
The selected defaults achieve the best reconstruction quality among the tested settings while maintaining practical runtime. 
Overall, these results show that our configuration is chosen based on empirical reconstruction quality and efficiency. 
Full results are provided in Appendix~\ref{app:hyperparam}.

\begin{figure}[H]
    \centering
    \includegraphics[width=\linewidth]{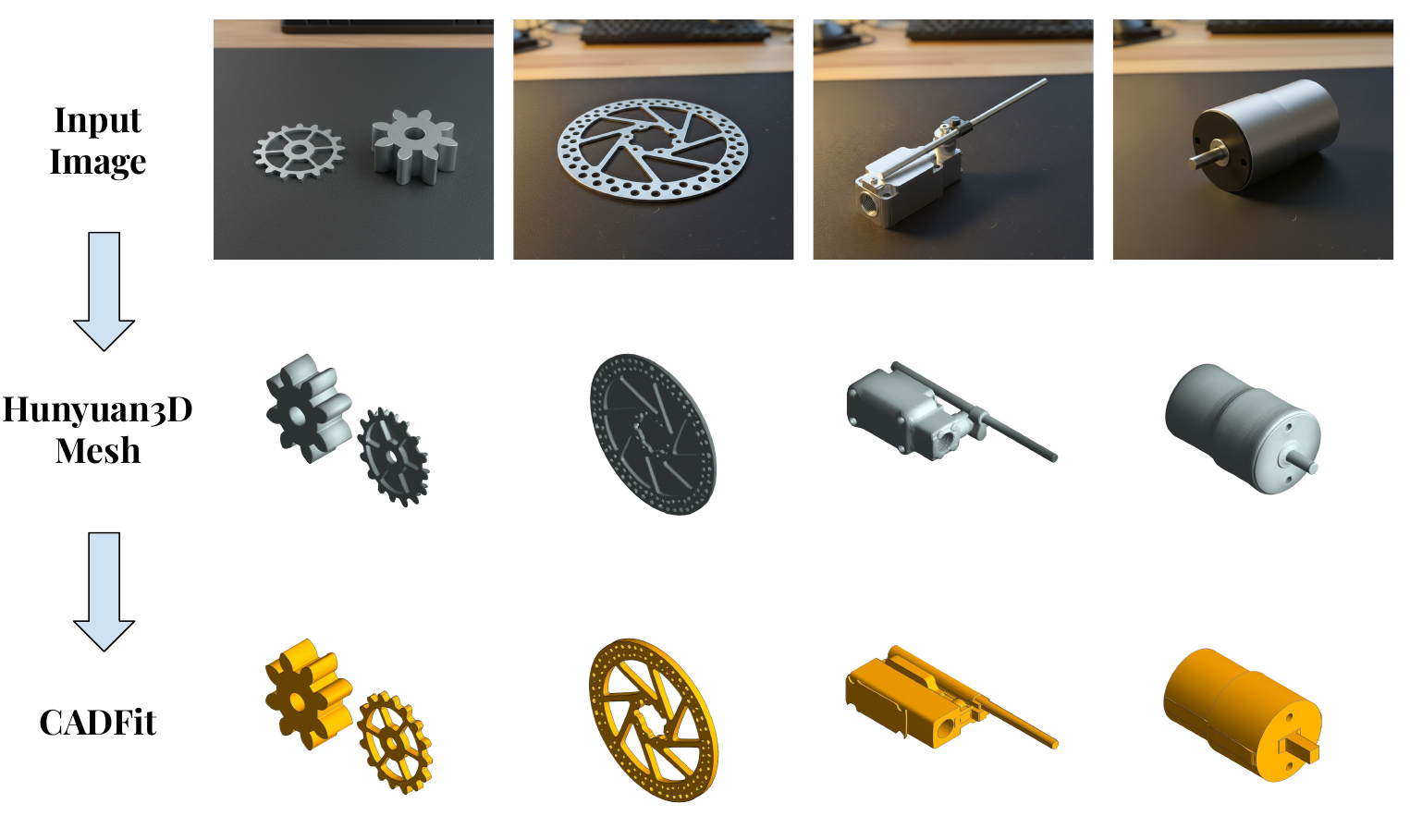}
    \caption{
    Photo-to-CAD reconstruction using Hunyuan3D and \textsc{CADFit}.
    Given realistic images of mechanical components, the pipeline reconstructs intermediate meshes and then recovers executable CAD construction sequences.
    }
    \label{fig:photo-to-cad}
\end{figure}

\textbf{Application: Photo-to-CAD Reconstruction.}
We show another application of CADFit by using it for photo-to-CAD reconstruction on realistic images of mechanical components.
We generate photo-style renderings from a component dataset~\cite{kim2020large} using Gemini~\cite{team2023gemini}, reconstruct intermediate meshes with Hunyuan3D, and recover executable CAD construction sequences with CADFit.
The results show that CADFit can recover structured programs beyond clean synthetic renderings, even when realistic inputs introduce mesh reconstruction artifacts.
Qualitative examples are shown in Fig.~\ref{fig:photo-to-cad}.

\section{Limitations}

CADFit trades inference speed for validity and reconstruction quality.
It is slower than feed-forward learned methods, which typically run in seconds, since CADFit requires minutes for iterative optimization and CAD-kernel validation (Appendix~\ref{app:more-results} for timings).
Its expressivity is also limited by the current operation set: extrusions, revolutions, fillets, and chamfers. This covers many engineering parts but can fail on shapes whose construction requires unsupported operations such as lofts, sweeps, or intersections. These unsupported operations also explain many of our observed failure cases. CADFit struggles with geometries dominated by smooth, high-curvature surfaces, since they often lack stable planar faces or slicing sections from which reliable sketches can be extracted (Appendix~\ref{app:failures}). Such failures are not simply optimization errors; they reflect a mismatch between the target shape and the sketch-and-operation vocabulary currently supported by CADFit. CADFit further depends on the quality of the input mesh.
Full residual refinement benefits from clean, watertight meshes because complement computation relies on robust boolean operations.
For noisy, highly tessellated, or non-watertight meshes from image-based reconstruction models, lightweight preprocessing such as smoothing and decimation is often needed to improve robustness and runtime.
Failures in upstream image-to-3D reconstruction can therefore propagate to the recovered CAD program.

\section{Conclusion}

We introduced \textbf{CADFit}, an optimization-based method for reconstructing structured, executable CAD construction sequences from geometric input.
By formulating reconstruction as iterative fitting with explicit program validation, CADFit reliably recovers complex parametric operations while guaranteeing validity.
Across mesh-to-CAD and image-to-CAD benchmarks, CADFit outperforms prior methods in geometric accuracy, robustness, and construction sequence quality, with especially large gains on high-complexity designs. CADFit adopts a decoupled design in which perception models recover surface geometry and parametric reconstruction is performed via geometry-driven optimization.
This allows the same pipeline to operate across meshes and images and enables improvements in upstream reconstruction to transfer directly to CAD inference.
Beyond precise geometric reconstruction, CADFit produces compact construction sequences that are often shorter than those in existing CAD datasets and prior methods. Future work includes extending the operator set to lofts, sweeps, pattern features and intersections, using CADFit to build higher-quality and complexity CAD datasets, and leveraging it as a structured policy for training learning-based or RL CAD agents with geometry feedback.

\newpage

\section*{Impact Statement}

This work aims to improve efficiency and structure in CAD reconstruction by recovering compact and precise construction programs from geometric observations. By enabling reusable and editable CAD representations, the proposed method may support faster engineering iteration and reduce reliance on manual modeling workflows. We do not foresee any immediate ethical or societal risks beyond those commonly associated with automation in engineering design tools.

\section*{Acknowledgments}
This work was supported by a gift from Altair to the DeCoDE Lab.


\bibliography{example_paper}
\bibliographystyle{icml2026}

\newpage
\onecolumn
\appendix

\section*{Table of Contents for Appendices}

\begin{enumerate}[label=\textbf{\Alph*.}, labelsep=1em]

\item \textbf{Additional Results and Visualizations} \hfill \pageref{app:more-results}

\item \textbf{Metric Computation} \hfill \pageref{app:metrics}

\item \textbf{Image-to-Mesh Preprocessing} \hfill \pageref{app:image2mesh}

\item \textbf{Ablation Study: Trellis vs Hunyuan3D} \hfill \pageref{app:trellis-hunyuan}

\item \textbf{Benchmark Construction} \hfill \pageref{app:benchmarks}

\item \textbf{Failure Cases} \hfill \pageref{app:failures}

\item \textbf{Theoretical Properties of IoU-Guided Assembly} \hfill \pageref{app:theory_cadfit}

\item \textbf{Hyperparameter Sensitivity and Robustness} \hfill \pageref{app:hyperparam}

\item \textbf{Sketch Extraction and Profile Formation} \hfill \pageref{app:sketchextraction}

\item \textbf{Sketch Primitive Fitting} \hfill \pageref{app:sketchprimitives}

\item \textbf{Geometry-Driven Operation Parameter Search} \hfill \pageref{app:paramsearch}

\item \textbf{IoU-Guided Construction Sequence Search} \hfill \pageref{app:assembly_algo}

\item \textbf{Learning-Based Sketch Prior} \hfill \pageref{app:learnedprior}

\item \textbf{Fillet and Chamfer Recovery} \hfill \pageref{app:fillet}

\end{enumerate}

\newpage

\section{Additional Results and Visualizations}
\label{app:more-results}
\begin{figure}[H]
    \centering
    \includegraphics[width=1\linewidth]{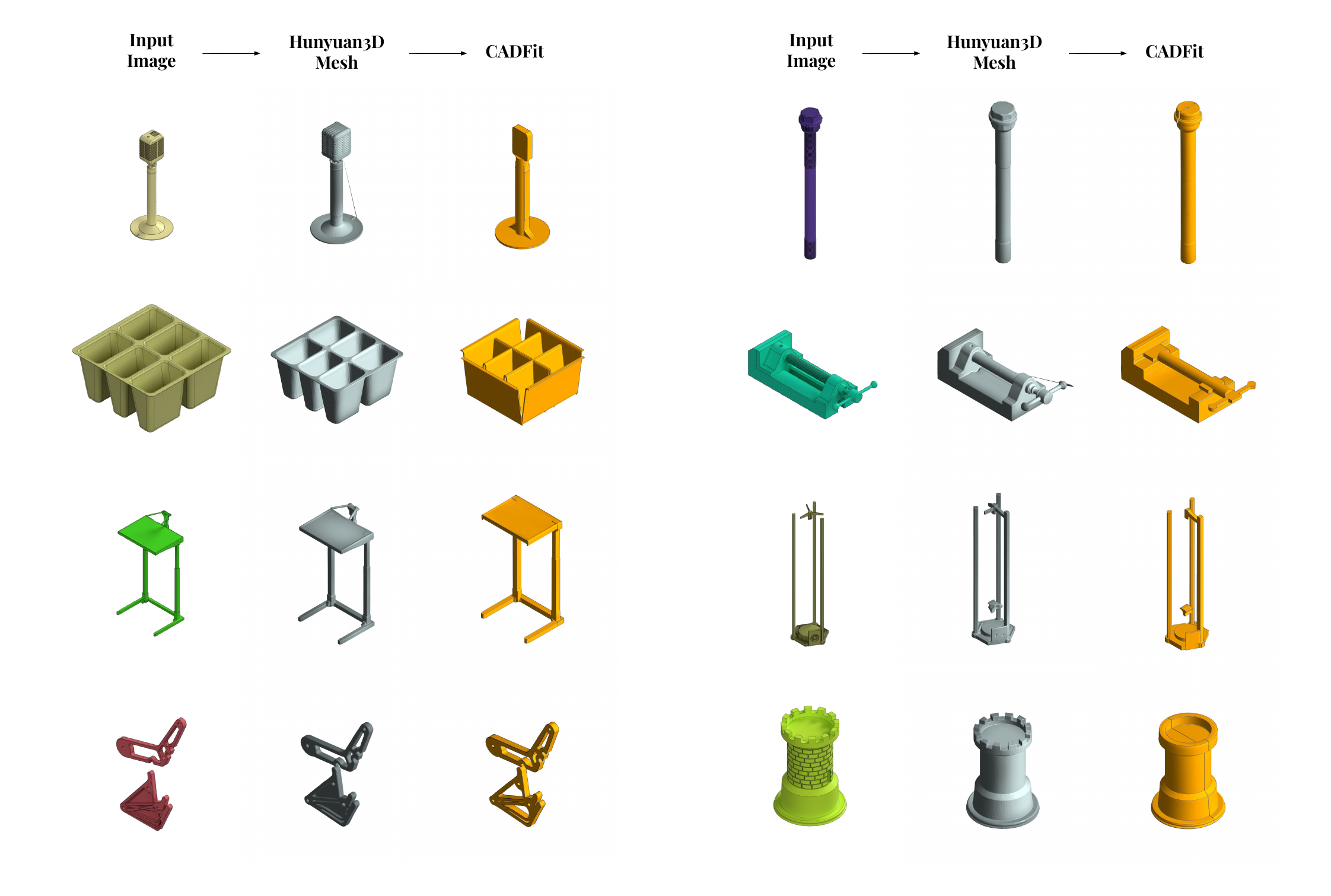}
    \caption{Qualitative Image-to-CAD results on the ABC Hard subset, showing the input image, intermediate Hunyuan3D mesh, and final CADFit reconstruction.}
    \label{fig:abc-hard-image-to-cad}
\end{figure}

\newpage

\begin{figure}[H]
    \centering
    \includegraphics[width=1\linewidth]{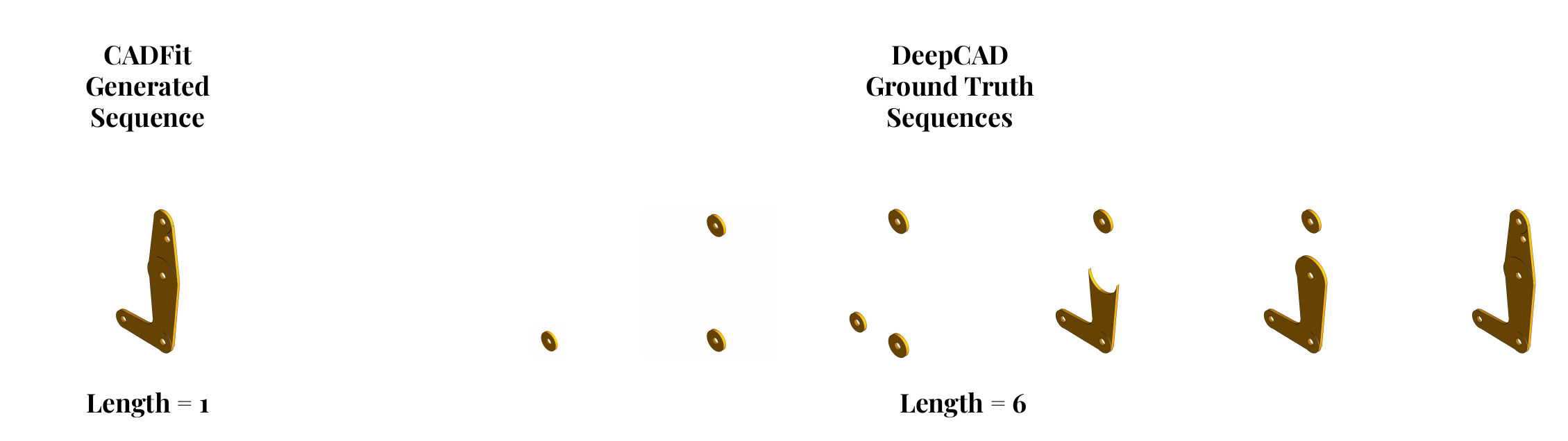}
    \caption{Example comparing \textsc{CADFit} output with DeepCAD ground-truth construction sequences. While DeepCAD programs may require multiple operations to form the target shape, \textsc{CADFit} frequently reconstructs the same geometry using fewer operations, yielding more concise CAD programs.}
    \label{fig:deepcad-cadfit}
\end{figure}

\begin{table}[H]
\centering
\caption{Median mesh-to-CAD reconstruction results on DeepCAD, Fusion360, and ABC stratified by CAD complexity.}
\label{tab:pc_mesh_results_median}
\resizebox{\textwidth}{!}{
\begin{tabular}{lcc cc cccccccc}
\toprule
 & \multicolumn{2}{c}{\textbf{DeepCAD}} 
 & \multicolumn{2}{c}{\textbf{Fusion360}}
 & \multicolumn{6}{c}{\textbf{ABC}} \\
\cmidrule(lr){2-3}
\cmidrule(lr){4-5}
\cmidrule(lr){6-11}

\textbf{Method}
& IoU$\uparrow$ & CD$\downarrow$
& IoU$\uparrow$ & CD$\downarrow$
& \multicolumn{2}{c}{Easy}
& \multicolumn{2}{c}{Medium}
& \multicolumn{2}{c}{Hard} \\
\cmidrule(lr){6-7}
\cmidrule(lr){8-9}
\cmidrule(lr){10-11}

 &  &  
 &  &  
 & IoU$\uparrow$ & CD$\downarrow$
 & IoU$\uparrow$ & CD$\downarrow$
 & IoU$\uparrow$ & CD$\downarrow$ \\
\midrule

GenCAD-3D \cite{yu2025gencad}
& 0.225 & 0.129
& 0.163 & 0.143
& 0.232 & 0.136
& 0.177 & 0.130
& 0.104 & 0.118 \\

CAD-Recode \cite{rukhovich2025cad}
& 0.900 & 0.018
& 0.874 & 0.017
& 0.885 & 0.019
& 0.627 & 0.029
& 0.267 & 0.043 \\

Cadrille \cite{kolodiazhnyi2025cadrille}
& 0.926 & 0.017
& 0.882 & 0.016
& 0.909 & 0.018
& 0.718 & 0.024
& 0.414 & 0.037 \\

\midrule
\textbf{CADFit (Ours)}
& \textbf{0.987} & \textbf{0.015}
& \textbf{0.979} & \textbf{0.013}
& \textbf{0.972} & \textbf{0.015}
& \textbf{0.908} & \textbf{0.020}
& \textbf{0.709} & \textbf{0.026} \\

\bottomrule
\end{tabular}
}
\end{table}

\begin{table}[H]
\centering
\caption{Average runtime (seconds per shape) of CADFit across benchmarks.}
\label{tab:runtime}
\begin{tabular}{lc}
\toprule
\textbf{Benchmark} & \textbf{Mean Runtime (s)} \\
\midrule
\multicolumn{2}{l}{\textbf{Mesh-to-CAD}} \\
DeepCAD (complex subset) & 242.3 \\
Fusion360 & 235.0 \\
ABC Easy & 303.0 \\
ABC Medium & 396.1 \\
ABC Hard & 903.7 \\
\midrule
\multicolumn{2}{l}{\textbf{Image-to-CAD (ABC)}} \\
ABC Easy (Hunyuan3D) & 299.2 \\
ABC Medium (Hunyuan3D) & 343.5 \\
ABC Hard (Hunyuan3D) & 357.9 \\
\bottomrule
\end{tabular}
\end{table}

\newpage

\section{Metric Computation}
\label{app:metrics}

\subsection{Shape Alignment}
\label{app:alignment}

Predicted CAD models may differ from ground-truth shapes by rigid transformations or global scale.
To ensure that evaluation reflects geometric fidelity rather than coordinate conventions, we align predicted shapes to the ground truth prior to metric computation.

Given point clouds $P_{\mathrm{gt}}$ and $P_{\mathrm{pred}}$ sampled from the ground-truth and predicted surfaces, respectively, we estimate a similarity transform
$(\mathbf{t}, \mathbf{R}, s)$ consisting of translation, rotation, and isotropic scale by minimizing symmetric Chamfer Distance:
\begin{equation*}
\min_{\mathbf{t},\mathbf{R},s}
\;
\mathrm{CD}\big(s\,\mathbf{R} P_{\mathrm{pred}} + \mathbf{t},\; P_{\mathrm{gt}}\big).
\end{equation*}
Optimization is performed using derivative-free Powell search.
The recovered transform is applied to the predicted mesh prior to computing CD and IoU.

\subsection{Chamfer Distance}
\label{app:cd}

Chamfer Distance is computed between surface point clouds sampled uniformly from the predicted and ground-truth meshes.
Specifically, given point sets $P$ and $\hat{P}$, we compute the symmetric Chamfer Distance
\begin{equation*}
\mathrm{CD}(P,\hat{P})
=
\frac{1}{2}
\left(
\frac{1}{|P|}\sum_{\mathbf{x}\in P}\min_{\mathbf{y}\in \hat{P}}\|\mathbf{x}-\mathbf{y}\|_2
+
\frac{1}{|\hat{P}|}\sum_{\mathbf{y}\in \hat{P}}\min_{\mathbf{x}\in P}\|\mathbf{x}-\mathbf{y}\|_2
\right).
\end{equation*}
All ground-truth meshes are normalized to the box $[-1,1]^3$, so reported CD values are directly comparable across shapes and datasets.

\subsection{Volumetric IoU Computation}
\label{app:iou}

Volumetric IoU is computed directly using Boolean operations between the predicted and ground-truth solids. 
Given two watertight meshes or CAD solids $A$ and $B$, we first compute their Boolean intersection and union:
\[
A \cap B,
\qquad
A \cup B.
\]
The IoU is then defined as the ratio between the volume of the intersected geometry and the volume of the union geometry:
\begin{equation*}
\mathrm{IoU}(A,B)
=
\frac{\mathrm{Vol}(A\cap B)}
{\mathrm{Vol}(A\cup B)}.
\end{equation*}

\newpage

\section{Image-to-Mesh Preprocessing}
\label{app:image2mesh}

For image-to-CAD experiments, meshes reconstructed from RGB images may contain noise, holes, or non-manifold artifacts that violate CADFit’s geometric assumptions.
We therefore apply a fixed preprocessing pipeline prior to CAD reconstruction.

Specifically, meshes produced by Hunyuan3D are post-processed as follows:
(i) watertightness is enforced using surface repair;
(ii) Taubin smoothing \cite{taubin1995curve} is applied to suppress high-frequency reconstruction noise while preserving overall shape;
(iii) mesh decimation is performed to obtain a well-conditioned triangulation with bounded face count.

All preprocessing hyperparameters are fixed across datasets and complexity levels.
A quantitative analysis of the impact of preprocessing on downstream CAD reconstruction is provided in this appendix.

\section{Ablation Study: Comparing Trellis and Hunyuan3D on image-to-mesh}
\label{app:trellis-hunyuan}

\begin{table}[H]
\centering
\caption{Image-to-mesh reconstruction quality on the ABC dataset for Hunyuan3D and Trellis.}
\label{tab:abc_hunyuan_trellis_stats}
\begin{tabular}{lcccc}
\toprule
\multicolumn{5}{c}{\textbf{ABC Easy}} \\
\cmidrule(lr){1-5}
\textbf{Method}
& IoU$_{\text{mean}}\uparrow$ & IoU$_{\text{median}}\uparrow$
& CD$_{\text{mean}}\downarrow$ & CD$_{\text{median}}\downarrow$ \\
\midrule
Hunyuan3D & 0.692 & 0.766 & 0.031 & 0.030 \\
Trellis   & 0.202 & 0.103 & 0.043 & 0.030 \\
\midrule
\multicolumn{5}{c}{\textbf{ABC Medium}} \\
\cmidrule(lr){1-5}
\textbf{Method}
& IoU$_{\text{mean}}\uparrow$ & IoU$_{\text{median}}\uparrow$
& CD$_{\text{mean}}\downarrow$ & CD$_{\text{median}}\downarrow$ \\
\midrule
Hunyuan3D & 0.602 & 0.680 & 0.031 & 0.026 \\
Trellis   & 0.189 & 0.152 & 0.050 & 0.040 \\
\midrule
\multicolumn{5}{c}{\textbf{ABC Hard}} \\
\cmidrule(lr){1-5}
\textbf{Method}
& IoU$_{\text{mean}}\uparrow$ & IoU$_{\text{median}}\uparrow$
& CD$_{\text{mean}}\downarrow$ & CD$_{\text{median}}\downarrow$ \\
\midrule
Hunyuan3D & 0.426 & 0.443 & 0.032 & 0.029 \\
Trellis   & 0.187 & 0.147 & 0.055 & 0.045 \\
\bottomrule
\end{tabular}
\end{table}

\newpage

\section{Benchmark Construction}
\label{app:benchmarks}

We construct evaluation benchmarks spanning a wide range of CAD complexity.
Complexity is measured using STEP face count, which correlates strongly with construction sequence depth and operation diversity.

For ABC, shapes are stratified into Easy, Medium, and Hard subsets based on face count thresholds, and samples are drawn uniformly over linear face count within each bin.
For DeepCAD and Fusion360, we sample from official test splits and publicly available galleries, respectively, following the protocol described in Section~\ref{sec:experiments}.

\begin{figure}[H]
    \centering
    \includegraphics[width=1\linewidth]{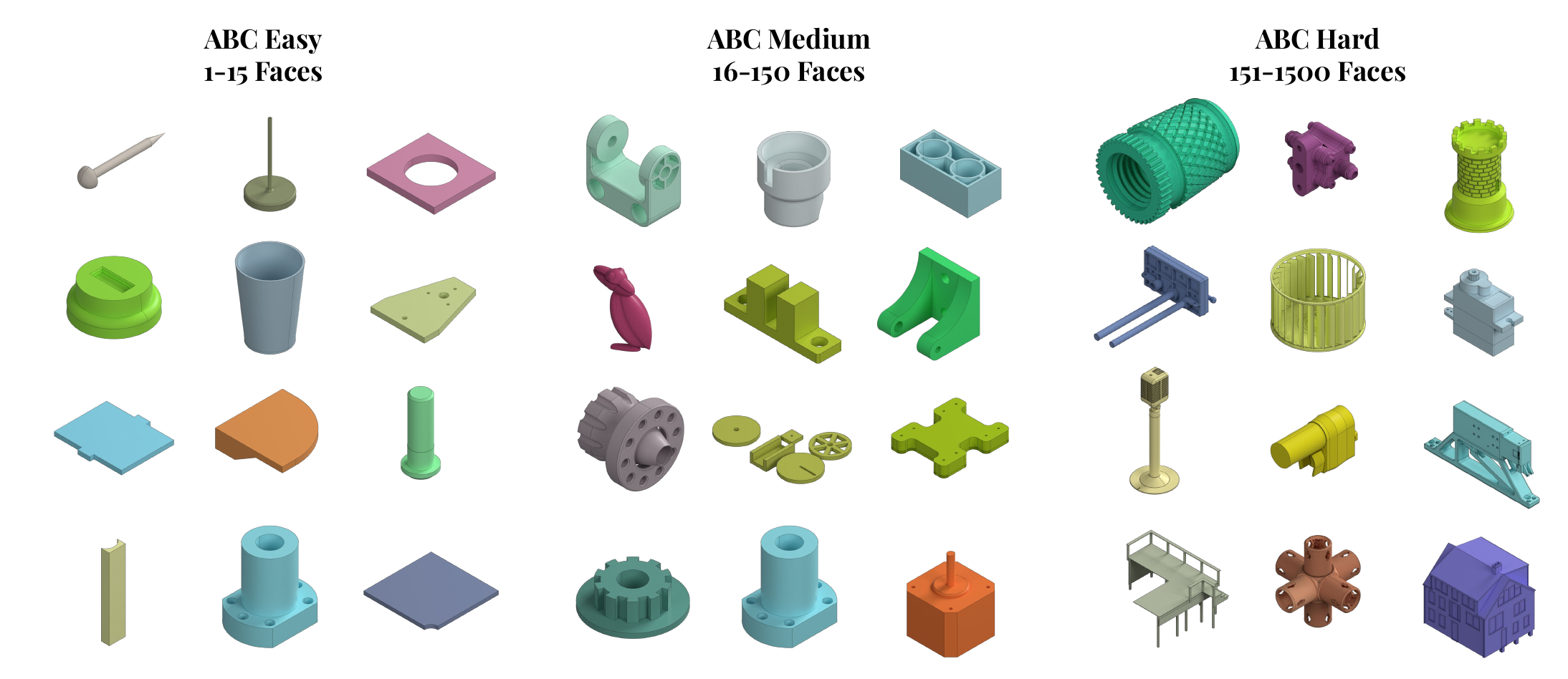}
    \caption{Qualitative visualizations of ABC CAD models stratified by complexity, rendered using Blender}
    \label{fig:abc-images}
\end{figure}
\newpage

\section{Failure Cases}
\label{app:failures}

\begin{figure}[H]
    \centering
    \includegraphics[width=1\linewidth]{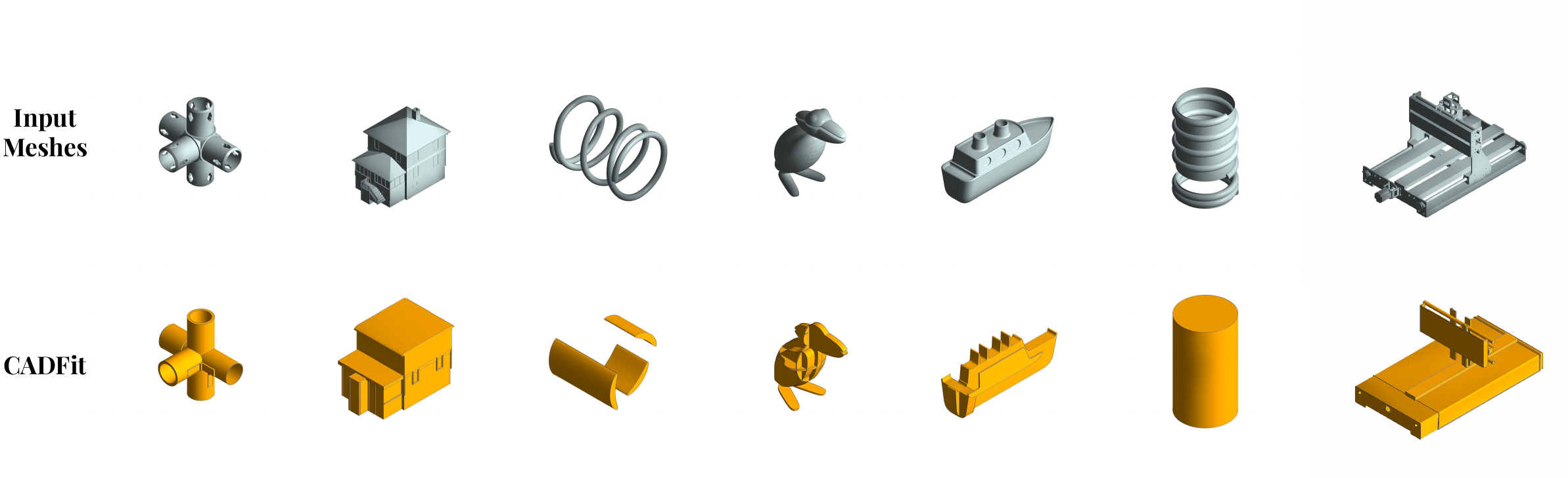}
    \caption{Failure cases of \textsc{CADFit} caused by unsupported operations and curved geometries.}

    \label{fig:failures}
\end{figure}

Figure~\ref{fig:failures} illustrates the main characteristic failure modes of \textsc{CADFit}.
Some failures arise from shapes whose construction relies on operations not currently supported by our kernel, such as lofts or sweeps.
Despite this limitation, \textsc{CADFit} often recovers close geometric approximations using combinations of supported primitives.

More challenging cases occur when the ground-truth geometry is dominated by smoothly curved surfaces, for which no stable sketch can be extracted from planar faces or slicing sections.
In such scenarios, the method struggles to recover precise constructions.
A promising direction to address these cases is to project the shape onto candidate planes and use the resulting contours as sketches, which could significantly extend coverage to curved geometries.

\newpage

\section{Theoretical Properties of IoU-Guided Program Assembly with \textsc{Union}/\textsc{Cut}}
\label{app:theory_cadfit}

This section provides a formal explanation for why CADFit’s IoU-guided residual assembly (i) monotonically improves geometric fidelity (in symmetric-difference error / IoU) up to modeling and candidate-resolution limits, and (ii) produces compact, irredundant programs via backward IoU pruning, under standard feature-based CAD assumptions.

\subsection{Setup}

Let $\mathcal{M}\subset\mathbb{R}^3$ be the target solid and let $S$ denote the current reconstructed solid. Define uncovered and overshoot residuals
\[
R^{+}(S) = \mathcal{M}\setminus S,
\qquad
R^{-}(S) = S\setminus \mathcal{M}.
\]
Define the normalized symmetric-difference error
\begin{equation}
E(S) \;=\; \frac{\mathrm{Vol}(\mathcal{M}\triangle S)}{\mathrm{Vol}(\mathcal{M})}
\;=\;
\frac{\mathrm{Vol}(R^{+}(S))+\mathrm{Vol}(R^{-}(S))}{\mathrm{Vol}(\mathcal{M})}.
\label{eq:E_def}
\end{equation}
Let
\[
a(S)=\frac{\mathrm{Vol}(R^{+}(S))}{\mathrm{Vol}(\mathcal{M})},\qquad
b(S)=\frac{\mathrm{Vol}(R^{-}(S))}{\mathrm{Vol}(\mathcal{M})},
\]
so that $E(S)=a(S)+b(S)$. IoU admits the identity
\begin{equation}
\mathrm{IoU}(S,\mathcal{M})
=
\frac{\mathrm{Vol}(S\cap\mathcal{M})}{\mathrm{Vol}(S\cup\mathcal{M})}
=
\frac{1-a(S)}{1+b(S)}.
\label{eq:iou_ab}
\end{equation}
Moreover, for all $a,b\ge 0$ with $a\le 1$,
\begin{equation}
1-\mathrm{IoU}(S,\mathcal{M})
=
1-\frac{1-a}{1+b}
=\frac{a+b}{1+b}
\le a+b
=E(S).
\label{eq:iou_to_E}
\end{equation}
Thus, decreasing $E(S)$ (symmetric-difference error) implies increasing IoU up to a monotone rescaling.

\subsection{Assumptions}

\textbf{Assumption A1 (Feature-based realizability up to a floor).}
There exists a feature-based CAD program whose supported operations (sketch-based \textsc{Extrude}/\textsc{Revolve} combined with Boolean \textsc{Union}/\textsc{Cut} and optional \textsc{Fillet}/\textsc{Chamfer}) can approximate $\mathcal{M}$ up to a modeling error floor $\varepsilon_{\mathrm{model}}\ge 0$:
\begin{equation}
\inf_{S\in\mathcal{S}_{\mathrm{CADFit}}}
\frac{\mathrm{Vol}(\mathcal{M}\triangle S)}{\mathrm{Vol}(\mathcal{M})}
\;\le\;
\varepsilon_{\mathrm{model}},
\label{eq:A1}
\end{equation}
where $\mathcal{S}_{\mathrm{CADFit}}$ denotes solids expressible by the operation set considered.

\textbf{Assumption A2 (Residual-aligned sketch availability).}
When run on the current residual geometry, the sketch extraction procedure proposes at least one plane/section whose induced candidate operations have nontrivial overlap with any residual region of sufficient volume. In particular, candidate generation is not required to be optimal, only to propose sketches that can ``touch'' the residual to enable improvement.

\textbf{Assumption A3 (Existence of a residual-improving candidate).}
At iteration $t$, for the current solid $S_t$, the candidate generator produces at least one update $(c_t,\op_t)$ with $\op_t\in\{\cup,\setminus\}$ such that
\begin{equation}
E(S_t \ \op_t\ c_t)
\;\le\;
\rho\,E(S_t) + \varepsilon_{\mathrm{cand}},
\qquad 0<\rho<1,\ \varepsilon_{\mathrm{cand}}\ge 0.
\label{eq:A3}
\end{equation}
Here $\varepsilon_{\mathrm{cand}}$ captures candidate-resolution limits (e.g., finite parameter sweeps) and $\rho$ captures a constant-factor progress condition.

\textbf{Assumption A4 (IoU-monotone acceptance).}
The update executed by CADFit at each iteration is chosen to have non-decreasing IoU, i.e.,
\[
\mathrm{IoU}(S_{t+1},\mathcal{M}) \ge \mathrm{IoU}(S_t,\mathcal{M}),
\]
equivalently (by~\eqref{eq:iou_ab}), it does not increase $E(\cdot)$ beyond monotone rescaling.

\subsection{Geometric Refinement with Overshoot Correction via \textsc{Cut}}

\begin{theorem}[Geometric refinement under IoU-guided residual assembly]
\label{thm:precision_cut}
Under Assumptions A1--A4, the residual procedure produces solids $\{S_t\}_{t=0}^T$ satisfying
\begin{equation}
E(S_T)
\;\le\;
\rho^T E(S_0) \;+\; \frac{\varepsilon_{\mathrm{cand}}}{1-\rho},
\label{eq:geom_conv_final}
\end{equation}
and consequently
\begin{equation}
\mathrm{IoU}(S_T,\mathcal{M})
\;\ge\;
1 - \rho^T E(S_0) - \frac{\varepsilon_{\mathrm{cand}}}{1-\rho}.
\label{eq:iou_conv_final}
\end{equation}
\end{theorem}

\begin{proof}
By Assumption A3, there exists at iteration $t$ a candidate update $(c_t,\op_t)$ such that
\[
E(S_t \ \op_t\ c_t) \le \rho E(S_t) + \varepsilon_{\mathrm{cand}}.
\]
By Assumption A4, CADFit selects an update with non-decreasing IoU, i.e., it does not worsen $E(\cdot)$ in the sense of the monotone relationship between IoU and $E$ given by~\eqref{eq:iou_ab}. Therefore, the executed update also satisfies
\[
E(S_{t+1}) \le \rho E(S_t) + \varepsilon_{\mathrm{cand}}.
\]
Unrolling the recurrence gives
\[
E(S_T) \le \rho^T E(S_0) + \varepsilon_{\mathrm{cand}}\sum_{k=0}^{T-1}\rho^k
= \rho^T E(S_0) + \varepsilon_{\mathrm{cand}}\frac{1-\rho^T}{1-\rho}
\le \rho^T E(S_0) + \frac{\varepsilon_{\mathrm{cand}}}{1-\rho},
\]
which proves~\eqref{eq:geom_conv_final}. Finally, applying~\eqref{eq:iou_to_E} yields~\eqref{eq:iou_conv_final}.
\end{proof}

\textbf{Why \textsc{Cut} corrects overshoot.}
Suppose a \textsc{Union} update adds extraneous volume $X \subset (c\setminus\mathcal{M})$. Then $X \subset R^{-}(S)$ by definition, i.e., it becomes overshoot residual. A subsequent \textsc{Cut} update $S\leftarrow S\setminus c'$ with $X\subseteq c'$ removes it and decreases $b(S)$, illustrating that overshoot introduced by unions is correctable through cuts, consistent with the residual formulation.

\subsection{Compactness via Backward IoU Pruning}

Let $\Pi$ denote the assembled program (sequence/multiset of operations, including whether each operation is applied as \textsc{Union} or \textsc{Cut}), and let $\mathrm{Solid}(\Pi)$ denote the CAD-kernel solid obtained by executing $\Pi$.
Backward IoU pruning repeatedly removes an operation whose removal yields the largest IoU gain (or smallest IoU drop), and terminates when any removal would decrease IoU.

\begin{lemma}[Irredundancy of the pruned program]
\label{lem:irredundant}
Let $\widehat{\Pi}$ be the program returned by backward IoU pruning and $\widehat{S}=\mathrm{Solid}(\widehat{\Pi})$.
Then for every operation $o\in\widehat{\Pi}$,
\[
\mathrm{IoU}(\mathrm{Solid}(\widehat{\Pi}\setminus\{o\}),\mathcal{M})
<
\mathrm{IoU}(\widehat{S},\mathcal{M}),
\]
i.e., removing any single operation strictly decreases IoU.
\end{lemma}

\begin{proof}
Backward pruning terminates only when no operation can be removed without decreasing IoU. Hence, at termination, every $o\in\widehat{\Pi}$ satisfies the strict inequality above.
\end{proof}

\begin{lemma}[Minimality among IoU-equivalent subsequences of the assembled program]
\label{lem:minimality}
Let $\widehat{\Pi}$ be the result of backward pruning applied to an initial assembled program $\Pi_{\mathrm{init}}$.
There is no strict subsequence $\Pi'\subset \widehat{\Pi}$ such that
\[
\mathrm{IoU}(\mathrm{Solid}(\Pi'),\mathcal{M})=\mathrm{IoU}(\mathrm{Solid}(\widehat{\Pi}),\mathcal{M}).
\]
\end{lemma}

\begin{proof}
Assume for contradiction that such $\Pi'$ exists. Then some operation $o\in\widehat{\Pi}\setminus\Pi'$ can be removed from $\widehat{\Pi}$ without changing IoU, contradicting the termination condition of backward pruning (Lemma~\ref{lem:irredundant}).
\end{proof}

\textbf{Discussion.}
Theorem~\ref{thm:precision_cut} formalizes progressive refinement: under a standard residual progress condition (A3), IoU-guided residual updates reduce symmetric-difference error geometrically down to a floor determined by modeling/candidate limits.
Lemmas~\ref{lem:irredundant}--\ref{lem:minimality} formalize compactness: backward IoU pruning removes all operations that do not contribute to the final IoU, producing an irredundant program that is minimal among IoU-equivalent subsequences of the assembled program.

\newpage

\section{Hyperparameter Sensitivity and Robustness}
\label{app:hyperparam}

We evaluate the sensitivity of \textsc{CADFit} to several key hyperparameters that affect sketch extraction, candidate generation, residual refinement, and parameter search. 
These ablations are conducted on held-out samples that are not included in the final test set, and are intended to justify the default configuration used throughout the main experiments.

Table~\ref{tab:hyperparam_ablation} reports reconstruction quality, measured by mean IoU, and runtime, measured as mean reconstruction time per shape. 
The default configuration achieves the highest IoU among the tested settings. 
Most nearby hyperparameter changes lead to moderate degradation rather than complete failure, suggesting that \textsc{CADFit} is not overly sensitive to small changes in individual parameters. 
Larger drops identify important design choices. 
In particular, using both axis-aligned slicing and planar face-based sketch extraction is important: restricting the sketch source to axis-aligned slices or planar faces alone substantially reduces IoU. 
Similarly, increasing the number of slices does not necessarily improve performance, since it can introduce many noisy or redundant sketch candidates that increase runtime and make pruning more difficult.

The Chamfer Distance threshold and translation step size show the expected quality--runtime trade-off. 
A smaller translation step gives similar IoU but increases runtime, while a larger step reduces runtime at the cost of reconstruction quality. 
The residual threshold also affects this trade-off: a looser threshold terminates earlier and reduces runtime, while a stricter threshold increases computation without improving quality in this setting. 
Overall, these results support the default configuration as a good balance between accuracy and runtime.

\begin{table}[H]
\centering
\caption{
Hyperparameter sensitivity of \textsc{CADFit} on held-out samples.
We report mean reconstruction IoU and mean runtime per shape. 
The default configuration, highlighted in bold, provides the best overall reconstruction quality while maintaining a reasonable runtime.
}
\label{tab:hyperparam_ablation}
\resizebox{0.7\linewidth}{!}{
\begin{tabular}{llcc}
\toprule
\textbf{Hyperparameter} & \textbf{Value} & \textbf{Mean IoU $\uparrow$} & \textbf{Mean Time (s) $\downarrow$} \\
\midrule

CD threshold
& \textbf{0.01 (base)} & \textbf{0.826} & \textbf{330.6} \\
& 0.005 & 0.594 & 271.8 \\
& 0.05  & 0.791 & 398.6 \\
\midrule

$\delta$ sketch extraction tolerance
& \textbf{0.05 (base)} & \textbf{0.826} & \textbf{330.6} \\
& 0.01 & 0.764 & 302.9 \\
& 0.10 & 0.744 & 314.2 \\
\midrule

$n_{\mathrm{slices}}$ axis-aligned slicing
& \textbf{5 (base)} & \textbf{0.826} & \textbf{330.6} \\
& 10 & 0.510 & 393.0 \\
& 3  & 0.764 & 242.1 \\
\midrule

Residual threshold
& \textbf{0.02 (base)} & \textbf{0.826} & \textbf{330.6} \\
& 0.01 & 0.752 & 340.6 \\
& 0.05 & 0.764 & 300.1 \\
\midrule

Sketch source
& \textbf{axis + planar (base)} & \textbf{0.826} & \textbf{330.6} \\
& axis only   & 0.668 & 202.9 \\
& planar only & 0.553 & 130.2 \\
\midrule

Translation step
& \textbf{0.01 (base)} & \textbf{0.826} & \textbf{330.6} \\
& 0.005 & 0.821 & 406.1 \\
& 0.02  & 0.684 & 257.4 \\
\bottomrule
\end{tabular}
}
\end{table}

\newpage

\section{Sketch Extraction and Profile Formation}
\label{app:sketchextraction}

This appendix details the geometric procedures used to extract sketch planes, section loops, and sketch profiles from the input mesh $\mathcal{M}$.
All steps are deterministic and independent of learning.

\begin{algorithm}[H]
\caption{Sketch Plane Extraction and Profile Grouping}
\label{alg:sketch_extraction}
\begin{algorithmic}
\STATE {\bfseries Input:} mesh $\mathcal{M}$
\STATE {\bfseries Output:} sketch profile set $\mathcal{G}$
\STATE Detect planar face clusters and axis-aligned slicing planes
\FORALL{sketch planes $(\mathbf{o},\mathbf{n})$}
    \STATE Intersect $\mathcal{M}$ with offset planes $\mathbf{o}\pm\delta\mathbf{n}$
    \STATE Extract closed section loops and project to plane
    \STATE Remove degenerate loops
    \STATE Group loops by polygon containment to form profiles
\ENDFOR
\STATE \textbf{return} $\mathcal{G}$
\end{algorithmic}
\end{algorithm}

\subsection{Sketch Plane Extraction}
Given a mesh $\mathcal{M}$, we construct a finite set of candidate sketch planes using two complementary strategies.

\textbf{Planar face-based planes.}
We cluster triangle normals using angular similarity and retain clusters whose total surface area exceeds a threshold.
Each cluster defines a plane with origin $\mathbf{o}_k$ (centroid of the cluster) and normal $\mathbf{n}_k$.
To robustly capture boundaries, we slice $\mathcal{M}$ using planes at offsets $\mathbf{o}_k \pm \delta \mathbf{n}_k$.

\begin{algorithm}[H]
  \caption{Planar Sketch Extraction}
  \label{alg:planar_sketch}
  \begin{algorithmic}
    \STATE {\bfseries Input:} mesh $M$
    \STATE {\bfseries Output:} planar sketch set $\mathcal{S}_{\mathrm{planar}}$

    \STATE Detect planar face clusters $\{(\mathbf{o}_k,\mathbf{n}_k)\}_{k=1}^K$ in $M$

    \FOR{$k = 1$ {\bfseries to} $K$}
        \STATE Define slicing planes at offsets $\mathbf{o}_k \pm \epsilon \mathbf{n}_k$
        \STATE Intersect $M$ with slicing planes to obtain cross-sections
        \STATE Extract closed boundary loops $\mathcal{L}_k$
        \IF{$\mathcal{L}_k \neq \emptyset$}
            \STATE Define sketch $s_k \leftarrow (\mathbf{o}_k,\mathbf{n}_k,\mathbf{u}_k,\mathcal{L}_k)$
            \STATE Add $s_k$ to $\mathcal{S}_{\mathrm{planar}}$
        \ENDIF
    \ENDFOR

    \STATE \textbf{return} $\mathcal{S}_{\mathrm{planar}}$
  \end{algorithmic}
\end{algorithm}

\textbf{Axis-aligned planes.}
To capture features not aligned with detected planar faces, we additionally introduce axis-aligned slicing planes.
For each axis $a\in\{x,y,z\}$, we select a fixed number of slice positions $\{q_i\}$ corresponding to quantiles of the mesh extent along $a$, and define slicing planes $(a=q_i)$.

\begin{algorithm}[H]
  \caption{Axis-Aligned Sketch Extraction}
  \label{alg:axis_sketch}
  \begin{algorithmic}
    \STATE {\bfseries Input:} mesh $M$
    \STATE {\bfseries Output:} axis-aligned sketch set $\mathcal{S}_{\mathrm{axis}}$

    \FORALL{axes $a \in \{x,y,z\}$}
        \STATE Select slice positions $\{q_i\}$ as quantiles of $M$ along axis $a$
        \FORALL{slice positions $q_i$}
            \STATE Slice $M$ with plane $(a = q_i)$
            \STATE Extract closed loops $\mathcal{L}_{i,a}$
            \STATE Remove degenerate or rectangular loops
            \IF{$\mathcal{L}_{i,a} \neq \emptyset$}
                \STATE Define sketch $s_{i,a}$
                \STATE Add $s_{i,a}$ to $\mathcal{S}_{\mathrm{axis}}$
            \ENDIF
        \ENDFOR
    \ENDFOR

    \STATE \textbf{return} $\mathcal{S}_{\mathrm{axis}}$
  \end{algorithmic}
\end{algorithm}

\subsection{Loop Extraction and Projection}
For each slicing plane $(\mathbf{o},\mathbf{n})$, we intersect $\mathcal{M}$ with nearby offset planes and extract closed polygonal loops.
Each loop is projected back onto the original sketch plane and represented in a local 2D coordinate system.
Loops are resampled to a fixed number of ordered points and filtered using area and minimum-length thresholds to remove degenerate cases.

\subsection{Profile Grouping by Containment}
Extracted loops may correspond to outer boundaries and interior holes.
We construct a polygon containment graph in the sketch plane and group each outer loop with its contained inner loops.
Each resulting group defines a sketch profile suitable for CAD operations.

\begin{algorithm}[H]
  \caption{Loop Grouping by Containment}
  \label{alg:loop_group}
  \begin{algorithmic}
    \STATE {\bfseries Input:} sketch set $\mathcal{S}$
    \STATE {\bfseries Output:} profile group set $\mathcal{G}$

    \FORALL{sketches $s \in \mathcal{S}$}
        \STATE Project all loops in $s$ to the sketch plane
        \STATE Construct polygon containment graph
        \STATE Identify outer loops and their contained inner loops
        \FORALL{outer loops}
            \STATE Form profile group $g$ with associated inner loops
            \STATE Add $g$ to $\mathcal{G}$
        \ENDFOR
    \ENDFOR

    \STATE \textbf{return} $\mathcal{G}$
  \end{algorithmic}
\end{algorithm}

\section{Sketch Primitive Fitting}
\label{app:sketchprimitives}

Each sketch profile consists of an ordered set of 2D points
\(
L^{2D}=\{(x_i,y_i)\}_{i=1}^N
\)
lying in a sketch plane.
We approximate $L^{2D}$ using a compact sequence of parametric primitives to obtain an executable CAD sketch.

Given a point subset $\mathcal{P}\subset\mathbb{R}^2$, we evaluate candidate fits for:
(i) line segments via least-squares fitting,
(ii) circles and arcs via algebraic circle fitting with RMS error constraints,
and (iii) splines as a fallback representation.
Segments are greedily grown until fitting error exceeds a tolerance relative to the loop scale.
Adjacent colinear primitives are merged to reduce fragmentation.

The output is a sequence of primitives $\mathcal{S}(L)$ that approximates the original loop within tolerance and preserves loop topology.
Pseudo-code for the full segmentation procedure is provided in Algorithm.

\section{Geometry-Driven Operation Parameter Search}
\label{app:paramsearch}

This appendix details the continuous parameter selection procedure used to generate discrete \textsc{Extrude} and \textsc{Revolve} candidates from a sketch profile.

\begin{algorithm}[H]
\caption{Geometry-Driven Candidate Generation}
\label{alg:paramsearch}
\begin{algorithmic}
\STATE {\bfseries Input:} sketch profile $g$, mesh surface points $Q(\mathcal{M})$
\STATE {\bfseries Output:} candidate set $\mathcal{C}_g$
\STATE Sweep extrusion heights and identify stable intervals
\STATE Canonicalize signed intervals and add extrusion candidates
\IF{profile $g$ is symmetric}
    \STATE Evaluate full revolutions about candidate axes
    \STATE Add best revolve candidate
\ENDIF
\STATE \textbf{return} $\mathcal{C}_g$
\end{algorithmic}
\end{algorithm}

\subsection{One-Sided Chamfer Distance}
Let $Q(\mathcal{M})=\{\mathbf{y}_j\}$ denote a surface point cloud sampled from $\mathcal{M}$.
For an operation $o(g,\theta)$ applied to sketch profile $g$ with parameters $\theta$, let $P(g,\theta)=\{\mathbf{x}_i\}$ denote points sampled from the induced surface.
We define the one-sided Chamfer Distance:
\begin{equation}
\mathcal{D}(\theta)
=
\frac{1}{|P(g,\theta)|}
\sum_{\mathbf{x}\in P(g,\theta)}
\min_{\mathbf{y}\in Q(\mathcal{M})}
\|\mathbf{x}-\mathbf{y}\|_2^2.
\end{equation}

\subsection{Extrusion Interval Selection}
For \textsc{Extrude}, the parameter $\theta=h$ denotes a signed extrusion height along the sketch normal.
We evaluate the Chamfer objective $\mathcal{D}(h)$ on a uniformly sampled, bounded interval $h\in[h_{\min},h_{\max}]$.
To identify geometrically valid extrusion ranges, we perform a monotonic error sweep and detect sharp increases in $\mathcal{D}(h)$ using finite-difference slope estimates.
Specifically, we mark a boundary at $h^\star$ when the discrete derivative
\[
\frac{\mathcal{D}(h_{i+1})-\mathcal{D}(h_i)}{h_{i+1}-h_i}
\]
exceeds a fixed threshold, indicating the onset of geometric overshoot.

Contiguous regions preceding such boundaries with $\mathcal{D}(h)\le\epsilon$ are treated as stable intervals and merged when adjacent.
From these, we retain two canonical candidates: (i) the smallest stable interval containing $h=0$ and (ii) the largest stable interval.
If a stable interval spans $h^-<0<h^+$, we reparameterize the extrusion by translating the sketch plane origin by $h^-\,\mathbf{n}$ and using height $h^+-h^-$.

\subsection{Revolution Axis Selection}
For \textsc{Revolve}, candidates are restricted to full $360^\circ$ revolutions.
We generate a finite set of axis hypotheses consisting of the principal axes of the loop point distribution and four additional axes rotated by $45^\circ$.
For each axis $a$, we evaluate $\mathcal{D}(a)$ and select the axis minimizing the Chamfer objective.

\section{IoU-Guided Construction Sequence Search}
\label{app:assembly_algo}

CADFit assembles generated operation candidates using a two-stage IoU-guided search procedure. 
First, it greedily selects candidates that improve IoU with the target mesh. 
Second, it applies backward pruning to remove redundant operations whose deletion does not reduce reconstruction quality. 
This produces a compact construction sequence while preserving geometric fidelity.

\begin{algorithm}[H]
  \caption{IoU-Guided Construction Sequence Search}
  \label{alg:search}
  \begin{algorithmic}
    \STATE {\bfseries Input:} sketch profiles $\mathcal{G}$, mesh $\mathcal{M}$
    \STATE {\bfseries Output:} CAD program $\Pi$, solid $S$
    \STATE Generate candidates $\mathcal{C}$ \COMMENT{Appendix~\ref{app:paramsearch}}
    \STATE $\Pi \leftarrow \emptyset$, $S \leftarrow \emptyset$
    \REPEAT
        \STATE $c^\star \leftarrow \arg\max_{c\in\mathcal{C}\setminus\Pi}
        \Big[
        \mathrm{IoU}(S\cup c,\mathcal{M})
        -
        \mathrm{IoU}(S,\mathcal{M})
        \Big]$
        \IF{$\mathrm{IoU}(S\cup c^\star,\mathcal{M}) > \mathrm{IoU}(S,\mathcal{M})$}
            \STATE $\Pi \leftarrow \Pi \cup \{c^\star\}$, $S \leftarrow S\cup c^\star$
        \ELSE
            \STATE \textbf{break}
        \ENDIF
    \UNTIL{no candidate improves IoU}
    \REPEAT
        \STATE $o^\star \leftarrow \arg\max_{o\in\Pi}
        \Big[
        \mathrm{IoU}(S^{-}_o,\mathcal{M})
        -
        \mathrm{IoU}(S,\mathcal{M})
        \Big]$
        \STATE \textbf{where} $S^{-}_o \coloneqq \mathrm{Solid}(\Pi \setminus \{o\})$
        \IF{$\mathrm{IoU}(S^{-}_{o^\star},\mathcal{M}) \ge \mathrm{IoU}(S,\mathcal{M})$}
            \STATE $\Pi \leftarrow \Pi \setminus \{o^\star\}$, $S \leftarrow S^{-}_{o^\star}$
        \ELSE
            \STATE \textbf{break}
        \ENDIF
    \UNTIL{no removable operation improves IoU}
    \STATE \textbf{return} $(\Pi,S)$
  \end{algorithmic}
\end{algorithm}

\section{Learning-Based Sketch Prior}
\label{app:learnedprior}

To reduce the number of sketch profiles considered during optimization, we train a binary classifier that predicts whether a candidate profile is likely to be used by the recovered CAD program.

\textbf{Training data.}
Given a recovered program $\Pi^\star$, profiles used by operations in $\Pi^\star$ are labeled positive, and remaining extracted profiles are labeled negative.
Labels are generated automatically from CADFit reconstructions, without human annotation.
We build the labeled pool from 16,719 STLs and use approximately $10^5$ labeled sketch profiles for training, with held-out test shapes split at the STL level.

\textbf{Model.}
The classifier computes a probability
\(
f(\mathcal{M}, g)\in[0,1]
\)
from a joint representation of the target mesh and candidate loop.
For each loop, we render multi-view images of the normalized mesh $\mathcal{M}$ with the loop $g$ overlaid, including an isometric view and a face-on view aligned with the loop plane.
These views are encoded with DINOv2 and concatenated with lightweight geometric loop features, such as area, perimeter, point count, bounding-box extent, plane parameters, and axis alignment.
The resulting feature vector is passed to a 3-layer MLP with nonlinear activations and a sigmoid output.
In our implementation, the MLP uses two hidden layers with GELU activations and dropout, followed by a final binary prediction layer.

\textbf{Usage.}
At inference time, each sketch profile is scored by $p=f(\mathcal{M},g)$.
CADFit uses these scores to bias profile selection toward loops that are more likely to contribute to the final reconstruction, replacing random loop sampling with a more informative candidate budget.

\section{Fillet and Chamfer Recovery}
\label{app:fillet}

Given a reconstructed solid $S$, we identify candidate edges for secondary features.
For each edge, we apply a small test \textsc{Fillet} or \textsc{Chamfer} operation and evaluate the resulting IoU.
If IoU improves, the feature parameter (radius or length) is refined via one-dimensional line search to maximize IoU with $\mathcal{M}$.


\end{document}